\pdfoutput=1
\documentclass{article}

\PassOptionsToPackage{numbers,compress}{natbib}

\usepackage[preprint]{neurips_2026}

\usepackage[utf8]{inputenc}
\usepackage[T1]{fontenc}
\usepackage{hyperref}
\usepackage{url}
\usepackage{booktabs}
\usepackage{amsfonts}
\usepackage{amsmath}
\usepackage{amssymb}
\usepackage{amsthm}
\usepackage{mathtools}
\usepackage{nicefrac}
\usepackage{microtype}
\usepackage[table,dvipsnames]{xcolor}
\usepackage{graphicx}
\usepackage{wrapfig}
\usepackage[ruled,vlined,linesnumbered,rightnl]{algorithm2e}
\usepackage{enumitem}
\usepackage{xspace}

\usepackage{Definitions}

\renewcommand{\le}{\leqslant}

\renewcommand{\ge}{\geqslant}

\DeclareMathOperator{\KL}{KL}
\newcommand{\scorer}{Q_\psi}
\newcommand{\dv}{\xb}
\newcommand{\dy}{\yb}
\newcommand{\masktoken}{\texttt{[MASK]}}
\newcommand{\blanktoken}{\texttt{[BLANK]}}
\newcommand{\method}{\texttt{FReDA}\xspace}

\definecolor{GainGreen}{RGB}{0,128,64}

\title{
Forward-Free Diffusion Language Models \\
with BPTT-Free Looped Refinement
}

\author{Haotian Sun \quad Rushi Qiang \quad Yuqian Zheng \quad Bo Dai\thanks{Corresponding author.} \\ Georgia Institute of Technology \\ \texttt{haotian.sun@gatech.edu, bodai@cc.gatech.edu} }

\begin{document}

\maketitle

\setlength{\abovedisplayskip}{1pt}
\setlength{\abovedisplayshortskip}{1pt}
\setlength{\belowdisplayskip}{1pt}
\setlength{\belowdisplayshortskip}{1pt}
\setlength{\jot}{1pt}
\setlength{\floatsep}{1ex}

\begin{abstract}
Diffusion language models generate text through iterative denoising, offering a powerful alternative to autoregressive generation.
However, discrete language spaces lack a natural neighborhood structure for defining effective perturbations, motivating artificial corruption schemes in the forward process.
Such prescribed forward processes often produce states that are mathematically convenient but misaligned with the drafts and errors encountered during generation, resulting in degraded sample quality.
To address this limitation, we propose \method, a forward-free diffusion language model that eliminates the need for a hand-designed forward process.
We formulate diffusion language modeling as recursive distribution refinement, in which model-generated drafts serve as implicit intermediate states and the learned refinement model progressively moves the draft distribution toward the target distribution.
Training detaches the preceding refinement passes and backpropagates only through the final pass, avoiding backpropagation through time (BPTT) across refinement iterations while retaining standard backpropagation within the final Transformer pass.
Concretely, \method refines drafts by proposing candidate sequences and either directly performing self-refinement or selecting among parallel candidates via Best-of-$N$ refinement.
With this design, \method is neighborhood-agnostic, sampler-adaptive, and compatible with flexible refinement parameterizations.
Extensive evaluations in the sub-8B regime show that \method-4B outperforms larger diffusion base models on reasoning and coding benchmarks, achieving absolute gains of up to 15\%, while reaching a 1.5--1.8$\times$ average speedup over diffusion baselines and scaling effectively with additional refinement computation.
\end{abstract}

\section{Introduction}
\vspace{-3pt}

\label{sec:intro}
Diffusion models have emerged as a powerful generative paradigm across diverse domains, including image synthesis~\citep{ddpm, LDM}, biological sequence generation~\citep{diff-bio, diff-protein}, and language modeling~\citep{SEDD, D3PM, MDLM}.
Unlike autoregressive~(AR) models, which generate sequences via a left-to-right factorization, diffusion models formulate generation as an iterative refinement process that transforms samples from a simple initial distribution toward the target distribution through learned backward transitions, and therefore have the potential to improve long-horizon text modeling and accelerate generation.

Conventional diffusion models typically prescribe a forward noising process that gradually perturbs data samples under a predefined noise schedule~\cite{ddpm, MD4} from the initial distribution.
This forward process defines the intermediate distribution path that the backward model learns to invert, thereby strongly shaping the difficulty of both training and inference~\citep{D3PM,iDDPM,design-path}.
In continuous domains, this paradigm is naturally supported by Gaussian perturbations, which exploit local neighborhoods under Euclidean geometry and preserve meaningful proximity between nearby states~\citep{ddpm, score-based}.
In discrete language modeling, however, there is no analogous canonical geometry for noise design.
Token indices carry no intrinsic metric, and semantic proximity is highly context-dependent~\citep{D3PM, ctx_dep}.
Consequently, the choice of initial distribution and forward corruption process becomes central to modeling language with diffusion: a poorly matched forward process may induce intermediate states that are misaligned with the drafts and errors encountered during generation, making the learned backward transitions less effective~\citep{SCUD, mdpo}.

Many structured perturbation approaches have been proposed for diffusion language modeling, including structured discrete transition kernels~\citep{multinomial_diffusion, ctdd, D3PM}, absorbing-state masking~\citep{MDLM,MD4}, discrete score-based objectives~\citep{SEDD, MDM_secret}, and scaling recipes adapted from autoregressive models~\citep{DiffuLLaMA,LLaDA, sdar, dream7b}.
However, discrete diffusion methods with predefined corruption paths still exhibit a significant performance gap relative to AR models in language modeling~\citep{D3PM,gulrajani2023likelihood,he2023diffusionbert,SEDD}, showing that without a natural neighborhood structure in discrete language spaces, the best perturbation design remains unclear.

Motivated by the difficulty of designing forward perturbations in discrete space, we ask whether diffusion language modeling can bypass a hand-designed forward corruption process altogether.
We propose \textbf{F}orward-F\textbf{Re}e \textbf{D}iffusion L\textbf{A}nguage Model~(\method) to answer this question.
We emphasize that diffusion language models recursively refine a distribution from an initial draft distribution and progressively close the gap to the target data distribution through learned reverse transitions.
Based on this view, \method uses drafts generated by the current model as implicit intermediate states, allowing the distributional path to be induced by the sampler itself rather than prescribed by an external noising schedule.
This design makes \method both {\bf neighborhood agnostic}, \ie, avoiding artificial token-level perturbations in discrete space, and {\bf sampler adaptive}, \ie, training each transition on the drafts and errors produced by the current sampler.
\method naturally supports flexible refinement parameterizations, and we instantiate it with two designs: {\bf 1)} self-refinement, which directly revises the model's own draft; and {\bf 2)} stochastic Best-of-$N$ refinement, which proposes multiple candidate sequences in parallel and uses an additional score head to select the most promising one.

The recurrent use of the same refinement backbone also places \method directly in the Looped Transformer family~\citep{dehghani2019universal,saunshi2025latentthoughts,geiping2025recurrentdepth}.
Specifically, \method defines a new draft-space form of Looped Transformer: each pass revises an explicit model-generated token draft and feeds the revised draft into the next pass, rather than carrying an internal latent state across loops.
Because every loop produces a candidate output, it can be supervised directly against the clean target.
During training, the preceding refinement passes are detached and gradients are applied only through the final pass, so \method avoids BPTT across refinement iterations while retaining ordinary backpropagation within the final Transformer pass.
This draft-space loop also preserves block-parallel token updates and naturally supports branching into parallel candidates for Best-of-$N$ refinement.
We discuss this connection further in Sec.~\ref{sec:looped_connection}.


We train \method at the 4B scale using a 10B-token continued-pretraining corpus, which is substantially smaller than the training budgets used by several competing diffusion baselines.
We evaluate it extensively across general knowledge, mathematical reasoning, and code generation benchmarks in the sub-8B regime.
Although substantially smaller than several competing diffusion models, \method-4B {\bf surpasses 7--8B} diffusion base models on four reasoning and coding benchmarks, achieving absolute improvements of 5--15\%, while remaining competitive on general knowledge tasks.
\method also establishes a stronger quality--speed Pareto frontier, requiring fewer forward passes to reach comparable accuracy and generating more tokens per forward pass than existing diffusion baselines, with an average speedup of 1.5--1.8$\times$.
Additional analyses further show that \method scales effectively with refinement computation, yielding consistent gains over refinement iterations.

Our main contributions are as follows:
(1) We propose \method, a forward-free diffusion language model that uses model-generated drafts as implicit intermediate states and learns to refine them through flexible parameterizations. The recurrent use of a shared refinement backbone makes \method a draft-space Looped Transformer whose training avoids BPTT across refinement iterations.
(2) We empirically evaluate \method across general knowledge, mathematical reasoning, and code generation benchmarks, demonstrating that \method-4B substantially outperforms 7--8B diffusion baselines and establishes a strong quality--speed Pareto frontier.

\section{Preliminaries}

\label{sec:prelim}

\paragraph{Discrete diffusion models}
Let $\dv=(x_1,\dots,x_L)\in\Vcal^L$ denote a discrete sequence over a finite vocabulary $\Vcal$, and let $p_0(\dv)$ denote the target data distribution.
A diffusion model defines generation through a learned backward process,
\begin{equation}
    p_\theta\rbr{\dv_{0:T}} = \textstyle p_T\rbr{\dv_T}\prod_{t=1}^{T}p_{\theta,t}\rbr{\dv_{t-1}\mid \dv_t},
    \label{eq:backward_joint}
\end{equation}
where $\dv_T$ is sampled from a simple initial distribution, $\dv_0$ denotes target samples, and $\dv_{1:T-1}$ are auxiliary intermediate variables. 
Without additional structure, these intermediate variables are latent, making it generally intractable to directly maximize the marginal likelihood for $p_\theta\rbr{\xb_0}$. Conventional diffusion models bypass this complication by prescribing a forward corruption process,
\begin{equation}
    q\rbr{\dv_{0:T}} = \textstyle p\rbr{\xb_0}\prod_{t=1}^{T}q_t\rbr{\dv_t\mid \dv_{t-1}},
    \label{eq:forward_joint}
\end{equation}
which defines an auxiliary joint distribution over clean data and corrupted states. 
Training can then be written as joint distribution matching, 
\begin{equation}
    \min_{\theta}\,\,\KL\!\rbr{
    q\rbr{\dv_{0:T}}\,\|\, p_\theta\rbr{\dv_{0:T}}
    }.
    \label{eq:joint_kl}
\end{equation}
For diffusion language modeling, a common choice is masked diffusion models~(MDM), where the forward process $q_t$ is absorbing-state masking corruption~\citep{MD4, MDLM},
\begin{equation}
    q_t\rbr{z_i\mid x_i}=\alpha_t \mathbf{1}\{z_i=x_i\}+ (1-\alpha_t)\mathbf{1}\{z_i=\masktoken\}, 
    \label{eq:mdm_forward}
\end{equation}
where each token is either kept unchanged or replaced by a mask token $\masktoken$ controlled by some predefined masking schedule $\alpha_t$.
The resulting masked diffusion objective reduces to a weighted denoising loss over masking-corrupted positions, 
\begin{equation}
    \Lcal_{\mathrm{MDM}}\rbr{\theta} =\textstyle
    \EE_{t,\dv_0,\dv_t\sim q_t(\cdot\mid \dv_0)}
    \sbr{\lambda_t\sum_{i:\dv_{t,i}=\masktoken}-\log p_\theta\rbr{\dv_{0,i}\mid \dv_t,t}
    },
    \label{eq:mdm_objective}
\end{equation}
where $\lambda_t$ is determined by the masking schedule. 

\paragraph{Complications in forward process design}
The objective in Eq.~\eqref{eq:mdm_objective} highlights the central role of the forward corruption process, which specifies both the training inputs $\dv_t$ and the distributional path $q\rbr{\dv_{0:T}}$ that the learned backward process learns to invert. 
Thus, the difficulty of learning and inference is strongly determined by how the forward process decomposes the gap between the initial distribution and the target distribution. 
While Gaussian perturbations provide a natural local geometry in continuous diffusion~\citep{ddpm,score-based}, discrete language modeling has no intrinsic analogue. 
For MDM with absorbing-state masking in Eq.~\eqref{eq:mdm_forward}, the induced neighborhood of token $a$ degenerates to $\Ncal_t(a)=\cbr{a,\masktoken}$, where each token can only remain unchanged or collapse to the shared absorbing state. Therefore, no direct token-to-token proximity is preserved~\citep{SCUD, mdpo}. 
As a result, the model is trained on mathematically convenient intermediate states that may be misaligned with the drafts and errors encountered during generation.

\section{\method: Forward-Free Diffusion Language Model}

\label{sec:method}
To address these challenges in discrete diffusion, we formulate a forward-free diffusion model in which refinement is learned directly from sampler-induced intermediate drafts. 
We outline the proposed \method below and defer detailed derivations and proofs to Appendix~\ref{app:proofs}.

\subsection{Diffusion via Recursive Marginal Refinement}

\paragraph{Draft distribution initialization}
We begin by defining an initial draft distribution $p_{\theta,T}$. 
Since the goal is to approximate the target data distribution $p_0\rbr{\dv}$, a natural design principle is to make $p_{\theta,T}\rbr{\dv}$ as close as possible to $p_0\rbr{\dv}$, thereby reducing the number of refinement iterations required at inference time. 
This leads to the initial matching objective
\begin{equation}
    \min_{\theta}\KL\!\rbr{p_0(\dv)\,\|\,p_{\theta, T}(\dv)}.
    \label{eq:init_kl}
\end{equation}
If $p_{\theta,T}=p_0$, generation can be performed in a single step and no refinement is needed. 
In practice, finite model capacity, factorized decoding, and imperfect optimization introduce an approximation gap between $p_{\theta,T}$ and $p_0$.

\paragraph{Recursive marginal refinement} 
To close this gap, we introduce refinement transitions 
\begin{equation}
    K_{ \theta, t}\rbr{\dv\mid \dv'}, \qquad\dv'\sim p_{\theta, t},\qquad t=T, T-1,\ldots,1,
    \label{eq:refiner_transition}
\end{equation}
where each transition maps samples from the current draft distribution toward the target distribution. 
Applying $K_{\theta,t}$ to the current marginal $p_{\theta,t}$ induces a refined marginal
\begin{equation}
     p_{\theta, t-1}(\dv)={\textstyle \int} K_{\theta,t}\rbr{\dv\mid \dv'}p_{\theta, t}(\dv')\,d\dv'.
    \label{eq:refined_marginal}
\end{equation}
Recursive application of these transitions yields a sequence of model-induced marginals
which progressively closes the gap to the target distribution $p_0$. 
Unlike conventional diffusion, this path is not prescribed by a hand-designed forward corruption process, but instead emerges from the learned refinement dynamics themselves.

\paragraph{Refinement learning objective}
Given the current draft marginal $p_{\theta,t}$, the ideal refinement transition should make the next marginal $p_{\theta,t-1}$ closer to $p_0$. 
Plugging Eq.~\eqref{eq:refined_marginal} into Eq.~\eqref{eq:init_kl} gives
\begin{equation}
    \min_{\theta}
    \KL\!\left(
        p_0(\dv)
        \,\middle\|\,
        {\textstyle \int}
        K_{\theta,t}(\dv\mid \dv')
        p_{\theta,t}(\dv')
        \,d\dv'
    \right).
    \label{eq:ideal_refinement}
\end{equation}
This objective requires marginalizing over drafts $\dv'$, which makes direct optimization difficult. Variational learning through evidence lower bound~(ELBO) is a potential choice~\citep{alias2017variational}, 
\begin{proposition}[Joint surrogate for marginal refinement]
\label{prop:joint_surrogate}
For any refinement kernel $K_{\theta,t}(\dv\mid\dv')$ and any coupling $\gamma_t(\dv,\dv')$ with marginals $p_0(\dv)$ and $p_{\theta,t}(\dv')$, we have
\begin{equation}
\begin{aligned}
    \KL\!\left(
        p_0(\dv)
        \,\middle\|\,
        {\textstyle \int}
        K_{\theta,t}(\dv\mid \dv')
        p_{\theta,t}(\dv')
        \,d\dv'
    \right)
    &\le
    \KL\!\left(
        \gamma_t(\dv,\dv')
        \,\middle\|\,
        K_{\theta,t}(\dv\mid\dv')p_{\theta,t}(\dv')
    \right).
    \label{eq:joint_surrogate_bound}
\end{aligned}
\end{equation}
The equivalency holds if $\gamma_t\rbr{\dv, \dv'} = p_0(\dv)q(\dv'|\dv)$, where $q(\dv'|\dv)=\frac{K_{\theta,t}(\dv\mid\dv')p_{\theta,t}(\dv')}{\int K_{\theta,t}(\dv\mid\dv')p_{\theta,t}(\dv') d\dv' }$.
\end{proposition}
However, the ELBO quickly becomes unaffordable in terms of computation and memory cost w.r.t. the model size increasing, due to the requirement of the auxiliary model for posterior $q(\dv'|\dv)$ approximation.
We hereby introduce a tractable and efficient coupling, which is inherited from the diffusion learning spirit, therefore, without explicit posterior approximation; meanwhile, without predefined forward perturbation, therefore, bypassing the presumed neighborhood in discrete space. 
Intuitively, we revisit the generation procedure: given the target $\dv$, the model is gradually refining the output $\dv'$ aiming towards $\dv$, which implies ${\gamma\rbr{\dv, \dv'}} =  {p_0\rbr{\dv}}q\rbr{\dv'|\dv}$ should reflect the reverse of such a generation procedure. 
Such a dependency can be easily induced through the current learned $p_{\theta, t}\rbr{\cdot}$. Concretely, given the context $\dy$ of $\dv$, which could be the prefix or random mask version of $\dv$ in unconditional generation or the corresponding prompts in conditional generation, we can easily construct $q(\dv'|\dv) = p_{\theta, t}\rbr{\dv'|\dy}$. For the detailed efficient implementation of the sampling operation, please refer to~\appref{app:implementation}. 
With our construction, the refinement kernel is pushing further towards the ground truth $\dv$ from current intermediate output $\dv'$, therefore, we obtain our objective,
\begin{equation}
    \min_{\theta} \mathcal{L}_{t}(\theta)
    :=
    -
    \mathbb{E}_{\dv\sim p_0, \dv'\sim p_{\theta,t}}\sbr{
    \log K_{\theta,t}(\dv\mid\dv')}.
    \label{eq:conditional_refinement_loss}
\end{equation}
This joint probability facilitates learning the marginal objective in Eq.~\eqref{eq:ideal_refinement} through the tractable surrogate established in Proposition~\ref{prop:joint_surrogate}.

After learning $K_{\theta,t}$, next draft marginal $p_{\theta,t-1}$ becomes available through Eq.~\eqref{eq:refined_marginal} and we repeat the procedure to further reduce the gap between the current draft distribution and the target distribution.
In \method, the joint $\gamma_t$ is constructed from model-generated drafts $\dv'\sim \ptil_{\theta,t}$ and clean targets $\dv\sim p_0$ under the same conditioning context, rather than from a predefined forward perturbation. 
Thus, the intermediate states used for learning are induced by the sampler itself.
This recursive marginal refinement formulation has two practical benefits.

\paragraph{Neighborhood-agnostic refinement}
By removing the explicit forward corruption process, \method does not require a predefined token-level neighborhood perturbation. 
This is especially important in discrete language spaces, where there is no canonical metric for deciding which token states should be treated as nearby or easily reversible.

\paragraph{Capacity-adaptive refinement}
In our construction, each refinement transition is trained to move the current draft distribution closer to the target distribution $p_0$. 
Since $K_{\theta,t}$ is constrained by finite model capacity and finite candidate search, each refinement step closes only the portion of the remaining distributional gap that the current parameterization can represent. 
As a result, the useful number of refinement steps adapts naturally to the expressiveness of the learned refinement model.

Furthermore, we also provide formal justifications for the good properties of recursive marginal refinement in Appendix~\ref{app:proofs}, showing that the joint surrogate upper-bounds the marginal refinement objective, preserves the target marginal at the optimum, and yields monotonic improvement under ideal refinement.

\subsection{Parameterization Design for Refinement Transitions}
\label{sec:parameterization_design}
The \method framework allows flexible parameterizations of the refinement transition $K_{\theta}$. In the following, we discuss two design choices for $K_{\theta}$.

\paragraph{Self-refinement.}
Given a current draft $\dv'$, the model predicts a clean sequence through a token-factorized conditional distribution, 
\begin{equation}
  K_{\theta}^\mathrm{SR}\rbr{\dv\mid \dv'}
    =\textstyle
    \prod_{i=1}^{L}
   K_{\theta}^\mathrm{SR}\rbr{x_i\mid \dv'},
    \label{eq:self_refine_transition}
\end{equation}
where the model conditions on the full current draft and outputs per-token refinement probabilities. Plugging \eqref{eq:self_refine_transition} into Eq.~\eqref{eq:conditional_refinement_loss} gives the self-refinement objective
\begin{equation}
   \min_{\theta}  \Lcal_{\mathrm{SR}}(\theta) :=-\EE_{t,\,\dv\sim p_0,\,\dv'_t\sim p_{\theta,t}}
    \sbr{\textstyle\sum_{i=1}^{L}\log K_{\theta}^\mathrm{SR}(x_i\mid \dv'_t)}.
    \label{eq:self_refine_loss}
\end{equation}
At inference time, the model repeatedly feeds its own predicted draft back into the same refinement operator: $\dv_{k+1} = \operatorname{Decode}\rbr{K_{\theta}^\mathrm{SR}(\cdot\mid y_k)}$, where $\operatorname{Decode}(\cdot)$ can be greedy decoding, sampling, or threshold-based commitment.  This yields a stochastic fixed-point refinement process in which high-quality samples should become increasingly stable as refinement proceeds.

\paragraph{Best-of-$N$ refinement.} 
The self-refinement parameterization in Eq.~\eqref{eq:self_refine_transition} follows a single refinement trajectory from each draft. 
To expand the search space, we further introduce a stochastic best-of-$N$ refinement parameterization. 
Given a current draft $\dv'$, the self-refinement proposer $K_{\theta,\mathrm{SR}}\rbr{\dv\mid\dv'}$ first generates $N$ candidate refined sequences,
\begin{equation}
    \cbr{\dv_i}_{i=1}^N\sim K_{\theta}^\mathrm{SR}\rbr{\cdot\mid \dv'}.
\end{equation}
We then augment the proposer with a learned scorer $\scorer\rbr{\cdot}$ that evaluates and selects the top candidate via a stochastic best-of-$N$ refinement step, 
\begin{equation}
    J=\arg\max_{j\in\sbr{N}}\cbr{\scorer\rbr{\dv_j}+\xi_j},    \qquad
    \xi_j\sim\mathrm{Gumbel}(0,1), \qquad
    \dv^\star=\dv_J.
    \label{eq:gumbel_bon}
\end{equation}
By the Gumbel-Max trick~\citep{maddison2014sampling}, the selected samples follows
\begin{equation}
 \dv^\star\sim \PP\rbr{\dv^\star=\dv_\ell\mid \cbr{\dv_i}_{i=1}^N}
    =
    \textstyle\frac{
        \exp\rbr{\tau\scorer\rbr{\xb_\ell}}
    }{
        \sum_{j=1}^{N}
        \exp\rbr{\tau\scorer\rbr{\xb_j}}
    },
    \label{eq:finite_candidate_softmax}
\end{equation}
where $\tau$ is the temperature factor. 
The finite-candidate selection rule induces a valid refinement parameterization as depicted in Proposition~\ref{prop:bon_transition}.
\begin{proposition}[Induced best-of-$N$ refinement transition]
\label{prop:bon_transition}
Without loss of generality, we assume $\dv_1$ corresponds to the best candidate. This parameterizes the refinement probability as
\begin{equation}
    \begin{aligned}
        &K_{\theta, \psi}^\mathrm{BoN}\rbr{\dv^\star=\dv_1\mid \xb'}\propto\\
        &
        K_{\theta}^\mathrm{SR}\rbr{\xb_1\mid\xb'} 
        \EE_{\xb_{2:N}\sim K_{\theta}^\mathrm{SR}\rbr{\xb_i\mid\xb'} }
        \sbr{\textstyle\frac{
        \exp\rbr{\tau\scorer\rbr{\xb_1}}
    }{
        \exp\rbr{\tau\scorer\rbr{\xb_1}}+\sum_{j=2}^{N}
        \exp\rbr{\tau\scorer\rbr{\xb_j}
    }}},
     \label{eq:refine_prob_bon}
    \end{aligned}
\end{equation}
\end{proposition}
which shows that best-of-$N$ refinement selects the self-refinement proposal toward candidates preferred by the learned scorer. 
Optimizing Eq.~\eqref{eq:conditional_refinement_loss} with the transition defined in Eq.~\eqref{eq:refine_prob_bon} leads to a tractable objective for best-of-$N$ refinement
\begin{equation}
  \min_{\psi}   \Lcal_{\mathrm{BoN}}\rbr{\psi}  := 
  \EE_{t,\,\xb_1\sim p_0,\, \cbr{\dv_i}_{i=2}^N\sim K_{\theta}^\mathrm{SR}\rbr{\cdot\mid \dv'}}\sbr{\textstyle
   - \scorer\rbr{\xb_1} + \log\sum_{j=1}^N\exp\rbr{\scorer\rbr{\xb_j}}
  }.
  \label{eq:bon_loss}
\end{equation}
During this stage, the proposal distribution $K_{\theta}^\mathrm{SR}\rbr{\cdot\mid\dv'}$ is treated with stop-gradient, so that $\Lcal_{\mathrm{BoN}}$ optimizes only the scorer head parameters. 
Through Eq.~\eqref{eq:bon_loss}, the scorer learns how much more likely a candidate is under the target distribution than under the self-refinement proposal, enabling best-of-$N$ decoding to select higher-quality refinements from parallel candidates.

\subsection{Discussions and Implementation}
\label{sec:discussion}
We discuss several key properties of \method and its connections to related refinement-based generative modeling frameworks.

\paragraph{Flexible inference}
The learned refinement chain in Eq.~\eqref{eq:conditional_refinement_loss} naturally induces a sampling procedure of the same form as the backward process in Eq.~\eqref{eq:backward_joint}. 
However, forward-free refinement also enables more flexible inference beyond a fixed-horizon diffusion sampler. 
Since each refinement stage is trained to move the current marginal $p_{\theta,t}$ closer to the target distribution $p_0$, any intermediate marginal can serve as an approximate generator when computation is limited. 
Conversely, once the learned refinement dynamics approach the target distribution, the transition approximately preserves $p_0$: $p_0(\dv_0')
    \approx
    {\textstyle \int}
    K_{\theta}\rbr{\dv_0'\mid \dv_0}
    p_0(\dv_0)
    \,d\dv_0,$
which indicates that $p_0$ is approximately stationary under the learned refinement transition, allowing additional refinement iterations beyond the training horizon. 
Thus, \method provides a controllable quality-cost tradeoff at inference time, where one can stop early for faster decoding or apply more refinement steps to improve sample quality.
Detailed empirical studies are provided in Sec.~\ref{sec:exp_inference}.

\paragraph{Connections to iterative refinement methods}
The proposed formulation is related to several broader ideas in generative modeling and search. 
Consistency models~\cite{consistency-model} learn direct maps from arbitrary noisy states $\xb_t$ to clean data $\xb_0$ and enforce agreement along a prescribed diffusion trajectory. 
\method shares the $\xb_0$-prediction perspective but eliminates the prescribed trajectory. This perspective casts \method as a forward-free consistency model, where consistency is induced by the fixed-point behavior of the learned refinement operator rather than imposed along a predefined diffusion trajectory. 
The proposed best-of-N design is also related to learning-to-search~\citep{learning-to-search}, since the refinement $K_{\theta, \psi}^\mathrm{BoN}$ learns to evaluate among model-generated candidate drafts under the same dynamics used at inference time. 
Finally, the marginal recursion in Eq.~\eqref{eq:refined_marginal} is conceptually related to power iteration over distributions, where repeated application of an operator moves a distribution toward a fixed point. 
Variational power methods~\citep{power-iteration} estimate stationary distributions of Markov chains by learning correction ratios from sampled transitions. 
From this view, the learned scorer acts as a density-ratio correction between target refinements and proposal candidates, allowing best-of-N refinement to reweight sampler-induced candidates toward the target distribution.

\paragraph{Key implementation choices}
We adopt several practical designs to make recursive refinement efficient and stable. 
\textbf{(1) Iteration-sampled training.} For each sample, we start from a blank sequence and run the proposer for $t$ refinement iterations, where $t\sim\mathrm{Unif}\rbr{\cbr{1,\ldots,K}}$ and $K$ is a hyperparameter. 
Gradients are stopped through the first $t-1$ refinement steps and applied only to the final step, covering both initial draft prediction and later-stage refinement with limited training overhead. 
\textbf{(2) Blockwise semi-causal attention.} For sequence generation, we use blockwise semi-causal attention with block size $B$, where token $i$ can attend to token $j$ if their block indices satisfy $b(j)\le b(i)$, allowing bidirectional attention within each block and causal attention across previous blocks, similar to block diffusion~\citep{block_diffusion}. 
The bidirectional within-block attention enables global refinement and scoring over the current block, while the causal structure across blocks preserves efficient left-to-right block generation. 
\textbf{(3) Efficient Best-of-$N$ scoring.} For Best-of-$N$ refinement, we generate high-quality candidates using marginal-cost beam search~(Pseudocode in Appendix~\ref{app:implementation}). 
Empirically, we select candidates using $\log K_{\theta,\mathrm{SR}}\rbr{\dv\mid\dv'}+\tau\scorer\rbr{\dv}$ to combine proposal likelihood with the learned correction and mitigate prediction noise in both proposer and scorer~(justifications for this design is available in Appendix~\ref{app:proofs}). 
\textbf{(4) Shared backbone.}
For efficiency, the proposer and scorer share the same backbone parameters, with only a lightweight scorer head added for Best-of-$N$ selection. 
Further details on the attention mask, training and inference pseudocode, scorer architecture, and hyperparameter choices are provided in Appendix~\ref{app:implementation} and Appendix~\ref{app:training_details}.

\subsection{Connection to Looped Transformers}
\label{sec:looped_connection}

\method can be understood directly as a Looped Transformer, which repeatedly reuses a shared Transformer module to refine a recurrent state, increasing effective computation without increasing the number of distinct model parameters~\citep{dehghani2019universal,saunshi2025latentthoughts,geiping2025recurrentdepth}.
In \method, the recurrent state is the model-generated block-level token draft itself, together with the confidence and commitment information used by the decoder.
The same refinement backbone is applied at every loop, and the resulting draft is supplied to the next pass.
Thus, \method realizes looped computation through explicit sequence refinement rather than through a persistent latent residual stream.

This choice makes supervision particularly simple.
Each loop state is already a candidate text output and can therefore be compared directly with the clean target.
Under iteration-sampled training, the earlier passes generate the draft under stop-gradient and only the final refinement pass receives gradients.
Accordingly, \method does not require BPTT across refinement iterations, although standard backpropagation is still performed within the final Transformer pass.
Additional loops are also not computationally free: reaching a deeper sampled draft requires additional forward-only refinement passes.
This detached training strategy is related to random deep supervision for latent Looped Transformers~\citep{park2026loopus}, but here it arises naturally from directly supervised token drafts.

The looped interpretation also clarifies the inference design.
Reusing the same backbone provides recurrent-depth scaling without increasing parameter count, while the explicit draft state preserves block-parallel token updates and can branch into multiple candidates for Best-of-$N$ selection.
Latent Looped Transformers may retain a richer hidden state across loops; \method instead restricts the recurrent state to a vocabulary-aligned draft in exchange for direct supervision, interpretable intermediate outputs, and simple branching.
We view these as complementary designs within the Looped Transformer family rather than claiming that one recurrent state parameterization is universally preferable.

\vspace{-10pt}
\section{Evaluations}
\vspace{-10pt}

\label{sec:experiments}

\begin{table}[t!]
\centering
\caption{
Performance comparison (\%) of AR and diffusion base models under 8B parameters across general knowledge, mathematics, and coding benchmarks. 
Numbers in parentheses denote the number of shots used for evaluation. 
Among diffusion base models, the best and second-best scores are highlighted in \textbf{bold} and \underline{underlining}, respectively. 
$^{*}$The DiffuLLaMA-7B GSM8K result is taken from the corresponding fine-tuned setting reported in the original source. 
``--'' indicates that the model is not open-sourced and the corresponding result is not reported, or meaningful generation could not be obtained.
}
\label{tab:main_results}
\setlength{\tabcolsep}{4.5pt}
\renewcommand{\arraystretch}{1.2}
\resizebox{\linewidth}{!}{
\begin{tabular}{l cc cc cc}
\toprule
\textbf{Tasks ($\rightarrow$)}
& \multicolumn{2}{c}{\textbf{General Knowledge}}
& \multicolumn{2}{c}{\textbf{Mathematics}}
& \multicolumn{2}{c}{\textbf{Coding}} \\
\cmidrule(lr){2-3}\cmidrule(lr){4-5}\cmidrule(lr){6-7}
\textbf{Method ($\downarrow$)}
& MMLU (5) & GPQA-D (5)
& GSM8K (4) & MATH-500 (4)
& HumanEval (0) & HumanEval+ (0) \\
\midrule
\rowcolor{gray!10} \multicolumn{7}{l}{\quad\emph{AR Base Models}} \\
\midrule
Qwen2.5-3B-Base~\citep{qwen25}
& 65.60 & 26.30
& 79.10 & 42.60
& 42.10 & 36.00 \\

Qwen3-4B-Base~\citep{qwen3}
& 71.76 & 33.68
& 84.99 & 52.60
& 65.24 & 57.93 \\

\midrule
\rowcolor{gray!10} \multicolumn{7}{l}{\quad\emph{Diffusion Base Models}} \\
\midrule
DiffuLLaMA-7B~\citep{DiffuLLaMA}
& 27.42 & 22.80
& 63.10$^{*}$ & 11.00
& -- & -- \\

LLaDA-MoE-7B-A1B-Base~\citep{llada-moe}
& 64.59 & 28.50
& 66.41 & 41.20
& 45.73 & 42.07 \\

Dream-7B-Base~\citep{dream7b}
& \underline{69.50} & \textbf{36.60}
& 81.00 & 39.20
& 57.90 & 50.00 \\

LLaDA-8B-Base~\citep{LLaDA}
& 65.90 & 25.20
& 74.37 & 39.60
& 35.40 & 31.10 \\

TiDAR-8B (Trust AR)~\citep{tidar}
& \textbf{76.57} & --
& 79.83 & --
& 55.49 & 52.44 \\

TiDAR-8B (Trust Diff)~\citep{tidar}
& -- & --
& 80.44 & --
& 57.93 & 55.49 \\

BlockDiff-4B~\citep{block_diffusion}
& 64.56 & 33.68
& 77.33 & 50.60
& 57.32 & 51.83 \\

\rowcolor{gray!15}
\method-4B (Self-refine)
& 69.11 & \underline{36.27}
& \underline{83.55} & \underline{51.98}
& \underline{60.98} & \underline{58.54} \\

\rowcolor{gray!20}
\method-4B (Best-of-N)
& 69.11 & \underline{36.27}
& \textbf{84.15} & \textbf{53.00}
& \textbf{63.41} & \textbf{59.76} \\
\bottomrule
\end{tabular}
}
\vspace{-3mm}
\end{table}

\paragraph{Training setups} We build \method on top of Qwen3-4B-Base~\citep{qwen3} and perform continued pretraining on a curated 10B-token data mixture containing several open-sourced datasets, which we decontaminated against all evaluation benchmarks. 
The model adopts blockwise causal attention with a block size of $4$, allowing bidirectional attention within each block while preserving causal attention across blocks. 
We implement this attention pattern with FlexAttention~\citep{flex_attn} to improve training efficiency.
Training proceeds in two phases:
1) We continue pretraining for 8B tokens with the self-refinement objective in Eq.~\eqref{eq:self_refine_loss}, with the budget of refinement forwards $K=3$;
2) We then branch into two refinement variants for the remaining 2B tokens. 
The self-refinement model continues to optimize the refinement objective, whereas the best-of-N model optimizes Eq.~\eqref{eq:bon_loss} using four model-generated negative candidates per instance~($N=4$). 
All training is conducted on 64 NVIDIA GH200 GPUs with a global batch size of 128.
Further implementation details are presented in Appendix~\ref{app:training_details}.

\paragraph{Baseline models}
We compare \method-4B against representative diffusion base models below 8B parameters, including \textbf{DiffuLLaMA}~\citep{DiffuLLaMA}, \textbf{LLaDA-MoE-7B-A1B-Base}~\citep{llada-moe}, \textbf{LLaDA-8B-Base}~\citep{LLaDA}, \textbf{Dream-7B-Base}~\citep{dream7b}, and \textbf{TiDAR-8B}~\citep{tidar}~(for which we report both the trusting-AR and trusting-diffusion variants). 
For a controlled efficiency comparison with masked-diffusion training, we further train a \textbf{BlockDiff-4B} baseline using the same training setups as \method, following the blockwise masked diffusion formulation in~\citep{block_diffusion}.

\paragraph{Evaluation benchmarks}
We evaluate \method and baselines on six benchmarks spanning \textbf{General Knowledge}~(MMLU~\citep{mmlu} and GPQA-Diamond~\citep{gpqa}), \textbf{Mathematical Reasoning}~(GSM8K~\citep{gsm8k} and MATH-500~\citep{math500}), and \textbf{Code Generation}~(HumanEval~\citep{humaneval} and HumanEval+~\citep{humaneval_plus}). 
All evaluations follow the standard shot settings for base models~\cite{qwen3} and report the final accuracy~(\%). 
As depicted in Sec.~\ref{sec:parameterization_design}, we investigate \method of two parameterization designs: \emph{Self-refinement} uses greedy decoding, and \emph{Best-of-N} generates candidates via greedy marginal-cost beam search and evaluates them by $\log K_{\theta,\mathrm{SR}}\rbr{\dv\mid\dv'}+\tau\scorer\rbr{\dv}$ with $\tau=0.1$.

\subsection{Main Results}

Table~\ref{tab:main_results} presents the decoding performance of \method and seven diffusion base models across general knowledge, mathematical reasoning, and code generation benchmarks. 
Overall, \method-4B achieves the strongest diffusion-model performance on four reasoning and coding benchmarks, despite being smaller than several 7-8B competitors. 
The two variants of \method validate the effectiveness of forward-free refinement. 
With best-of-N refinement, \method improves over the strongest prior diffusion baselines by at least 2.60\% on mathematical reasoning and 4.27\% on code generation, and delivers gains of roughly 5-15\% over representative larger diffusion models on most reasoning and coding tasks. 
Notably, \method achieves these gains with only a 10B-token continued-pretraining budget, whereas strong AR-adapted diffusion baselines such as Dream-7B-Base~\citep{dream7b} and TiDAR-8B~\citep{tidar} reportedly use 600B and 50B tokens, respectively.
Self-refinement also consistently outperforms existing diffusion baselines by at least 4.37\% on reasoning and coding tasks, showing that recursive draft refinement alone is already effective. 
Under the same initial model and training data, \method also surpasses the BlockDiff-4B baseline trained with a masked-diffusion objective by 5.06\% on average across all evaluated benchmarks~\citep{block_diffusion}. 
Notably, \method-4B also matches or exceeds the Qwen3-4B AR base model across most of the evaluated benchmarks, suggesting that forward-free refinement can substantially narrow the gap between diffusion and autoregressive language models.

\subsection{Inference Efficiency}
\label{sec:exp_inference}
\begin{figure}[!t]
    \centering
    \includegraphics[width=\linewidth]{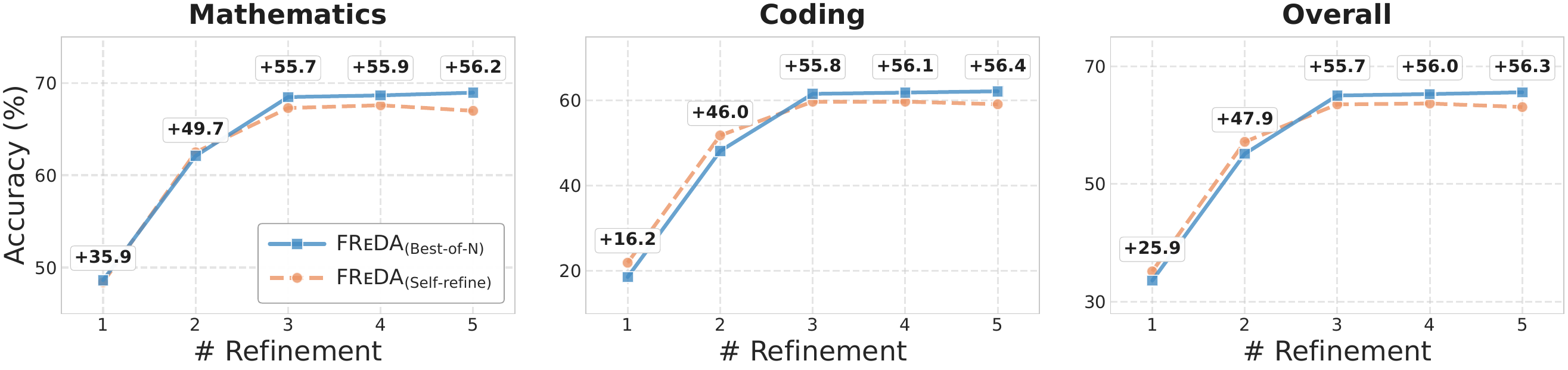}
    \caption{%
        \textbf{Iterative refinement continuously improves \method.}
        Accuracy~(\%) of \method~(Best-of-N) and \method~(Self-refine) with increasing number of refinement iterations over initial draft on Mathematics, Coding, and Overall. 
        We report the gain over the single-iteration baseline of the better variant at each iteration.     
        For refinement iterations 4 and 5, early stopping is disabled to isolate the effect of additional refinement.
    }
    \label{fig:iter-scaling}
\end{figure}
\vspace{-10pt}

\paragraph{Improvement through refinement}
As discussed in Sec.~\ref{sec:discussion}, \method supports flexible inference by applying additional refinement iterations after the initial draft prediction. 
Figure~\ref{fig:iter-scaling} shows how performance evolves as the number of refinement steps increases. 
Across benchmarks of different domains, both parameterizations of \method substantially improve over the single-step draft baseline, with overall accuracy increasing from 25.9\% at one refinement step to 56.3\% after five steps. 
The gains are most pronounced in the early refinement stage: moving from one to three steps improves overall accuracy by nearly 30\%, after which performance begins to converge. 
This convergence is consistent with a block size of $4$, in which most refinable positions within each block are revisited after a small number of iterations.
Comparing the two parameterizations, self-refinement is slightly stronger in the low-iteration regime, suggesting that greedy refinement can provide an efficient trajectory when only limited refinement computation is available. 
As the refinement budget increases, best-of-N becomes more effective and benefits from parallel candidate exploration and scorer-based selection. 
Overall, \method exhibits consistent self-improvement through iterative refinement and provides a controllable quality-compute tradeoff at inference time.

\begin{figure}[!t]
    \centering
    \includegraphics[width=\linewidth]{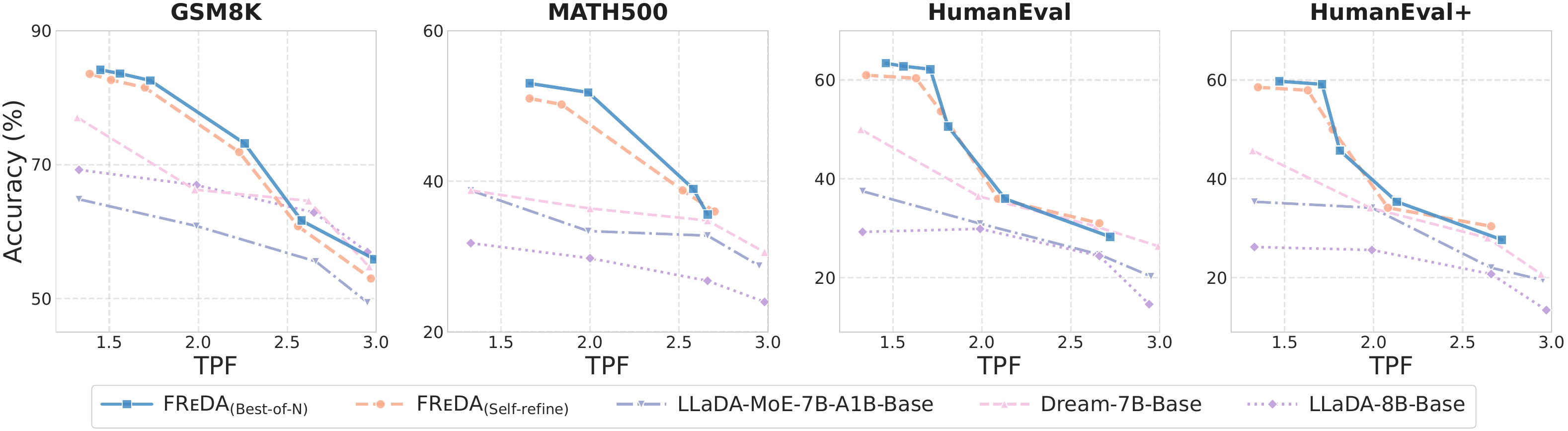}
    \caption{%
\textbf{\method Pareto frontier outperforms open diffusion baselines across math and coding tasks.}
Accuracy~(\%) of \method~(Best-of-N) and \method~(Self-refine) against tokens-per-forward~(TPF) on GSM8K, MATH-500, HumanEval, and HumanEval+.
The two \method curves are the Pareto envelopes of a joint sweep over the number of iterations and the confidence threshold for early stop; diffusion baselines follow each model's native TPF schedule.}
    \label{fig:pareto-frontier}
\end{figure}

\paragraph{Speed--quality Pareto frontier}
A key feature of diffusion language models is their ability to generate multiple tokens per forward pass, often at the expense of sample quality. 
Figure~\ref{fig:pareto-frontier} compares the speed-quality Pareto frontiers of \method and diffusion baselines. 
Across a wide range of tasks, both variants of \method achieve 1.5-2.5 tokens per forward while delivering at least 10\% higher decoding accuracy. 
Moreover, \method matches the best performance of existing diffusion baselines with fewer forward passes, yielding an average generation speedup of 1.5-1.8$\times$ at matched quality.
Even in the high-throughput regime with more than 2.5 tokens per forward pass, \method remains competitive with, and often surpasses, the diffusion baselines. 
These results highlight that forward-free refinement improves not only generation quality but also the quality-speed trade-off over existing masking diffusion language models.
\subsection{Ablation Studies}

\begin{wraptable}{r}{0.43\textwidth}
\vspace{-1.0em}
\centering
\caption{
Ablation of two proposed decoding strategies. 
Best-of-N selects candidates using $\log K_{\theta,\mathrm{SR}}\rbr{\dv\mid\dv'}+\tau\scorer\rbr{\dv}$. 
}
\label{tab:alpha_ablation}
\renewcommand{\arraystretch}{1.12}
\resizebox{\linewidth}{!}{
\begin{tabular}{l c ccc}
\toprule
\textbf{Strategy} 
& $\boldsymbol{\tau}$
& \textbf{Math}
& \textbf{Coding}
& \textbf{Overall} \\
\midrule
Self-refine
& -- 
& 67.77
& 59.76
& 63.76 \\

\midrule
Best-of-N
& 0.1
& {68.58}
& \textbf{61.59}
& \textbf{65.08} \\

Best-of-N
& 0.3
& {68.58}
& {61.28}
& {64.93} \\

Best-of-N
& 1.0
& \textbf{68.62}
& 60.07
& 64.34 \\
\bottomrule
\end{tabular}
}
\vspace{-1.0em}
\end{wraptable}
\paragraph{Temperature factor in Best-of-N scorer}
Table~\ref{tab:alpha_ablation} ablates the temperature $\tau$ in Best-of-N selection, where candidates are scored by $\log K_{\theta,\mathrm{SR}}\rbr{\dv\mid\dv'}+\tau\scorer\rbr{\dv}$. 
As shown in Sec.~\ref{sec:parameterization_design}, the scorer is trained to estimate a correction to the backbone likelihood, corresponding to the log-probability gap between the target  distribution and the proposal distribution. 
With $\tau=0.1$, Best-of-N achieves the best overall performance and improves the average score of self-refinement by 1.32\%.
This suggests that the scorer is most effective as a calibrated correction signal rather than a dominant ranking objective. 
Too large weights preserve similar mathematical performance but degrade coding and overall accuracy, suggesting that over-amplifying the scorer can introduce ranking noise. 
A reasonably small weight likely flips candidate preferences only when the estimated refinement gap is large, while retaining the backbone likelihood as the primary signal when scorer estimates are uncertain.

\begin{wraptable}{r}{0.45\textwidth}
\centering
\caption{
Ablation on refinement iterations and Best-of-N width. 
Iter/Width denotes the number of refinement iterations and the number of Best-of-N candidates. 
}
\label{tab:bon_width_ablation}
\setlength{\tabcolsep}{4.0pt}
\renewcommand{\arraystretch}{1.12}
\resizebox{\linewidth}{!}{
\begin{tabular}{c ccc}
\toprule
\textbf{Iter/Width}
& \textbf{Math}
& \textbf{Coding}
& \textbf{Overall} \\
\midrule
$2/2$
& 48.13
& 18.60
& 33.36 \\
$2/4$
& 48.82\,{\small\textcolor{ForestGreen}{(+0.70)}}
& 19.21\,{\small\textcolor{ForestGreen}{(+0.61)}}
& 34.01\,{\small\textcolor{ForestGreen}{(+0.65)}} \\
\midrule
$4/2$
& 68.14
& 60.98
& 64.56 \\
$4/4$
& 68.58\,{\small\textcolor{ForestGreen}{(+0.44)}}
& 61.59\,{\small\textcolor{ForestGreen}{(+0.61)}}
& 65.08\,{\small\textcolor{ForestGreen}{(+0.52)}} \\
\bottomrule
\end{tabular}
}
\vspace{-1.0em}
\end{wraptable}
\paragraph{Search width in Best-of-N}
Table~\ref{tab:bon_width_ablation} studies the effect of increasing the Best-of-N search width from 2 to 4 candidates under different refinement budgets. 
Although the gains are modest, increasing the width consistently improves performance across mathematics, coding, and overall averages. 
With two refinement iterations, the larger candidate set improves the overall score by 0.65\%, and with four iterations, it further improves it by 0.52\%. 
This suggests that broader candidate exploration provides a reliable complementary benefit to iterative refinement, and \method can exploit additional parallel inference computation by using a wider Best-of-N search.

\paragraph{Model parameterization}
We compare the two refinement parameterizations introduced in Sec.~\ref{sec:parameterization_design}. 
In Table~\ref{tab:main_results}, Best-of-N achieves better final performance than self-refinement. This improvement comes from its ability to explore multiple candidates in parallel and use the learned scorer to select more reliable refinements. 
Hypothesically, the scorer may provide a useful correction signal beyond the base model's likelihood, helping rectify generation-format errors in chain-of-thought reasoning and select syntactically or functionally better code snippets.
This advantage is also reflected in Figure~\ref{fig:pareto-frontier}, where Best-of-N yields a slightly stronger Pareto frontier in the medium-throughput regime. 
On the other hand, self-refinement is more effective under limited refinement budgets, as shown in Figure~\ref{fig:iter-scaling}, and often performs better in the high-throughput regime above 2.5 TPF. 
Overall, self-refinement is preferable for fast decoding, while Best-of-N provides stronger quality when additional candidate exploration is affordable.

\section{Related Work}

\label{sec:related}
\paragraph{Discrete diffusion language models}
Discrete diffusion models adapt the iterative denoising paradigm to categorical spaces.
Early work on discrete diffusion studies how to design corruption kernels for categorical data, including uniform, structured, nearest-neighbor, and absorbing-state transitions~\citep{D3PM,multinomial_diffusion,ctdd}.
In language modeling, absorbing-state masking has become a dominant template, in which models learn to recover clean tokens from partially masked sequences~\citep{MDLM,MD4}.
Other formulations improve or reinterpret this paradigm through discrete score estimation, conditional-distribution modeling, or importance-sampling-based refinement of the reverse process~\citep{SEDD,MDM_secret,edlm, DFM}.
Recent scaling efforts further show that diffusion language models can be trained or adapted at scaled model sizes up to 100B, substantially improving generation quality and downstream performance~\citep{DiffuLLaMA,llada-moe,LLaDA,dream7b,sdar,dmax,llada20,llada21}.
Despite this progress, most methods still define learning through a prescribed forward corruption process.
\method departs from this paradigm by learning refinement directly from sampler-induced intermediate states.
\vspace{-4pt}

\paragraph{Looped transformers}
Looped transformers reuse shared parameters across depth to iteratively refine a recurrent state, beginning with Universal Transformers~\citep{dehghani2019universal} and extending to recent reasoning-oriented and recurrent-depth language models~\citep{saunshi2025latentthoughts,geiping2025recurrentdepth}.
\method belongs to this family but uses the model-generated token draft itself as the recurrent state.
This draft-space design makes each loop directly supervisable by the clean target, allowing preceding loops to be detached and avoiding BPTT across refinement iterations.
LoopUS similarly uses random deep supervision to train detached latent refinement trajectories~\citep{park2026loopus}; \method differs by coupling looped computation with sampler-induced diffusion refinement and block-parallel draft generation.
\vspace{-4pt}

\paragraph{Refinement-based language generation}
Drafting and refinement have been studied as alternatives to purely left-to-right AR generation.
Edit-based generation methods revise sequences through insertion, deletion, or iterative rewriting, providing early evidence that language generation can be organized as refinement rather than strict autoregression~\citep{levenshtein,diffuser}.
Corrector-style samplers further model refinement as a reverse causal process, in which tokens are corrected while conditioning on already generated future context~\citep{correctorar}.
In diffusion language models, several methods augment masked-diffusion backbones with token-level correction or remasking heads that decide which positions should be revisited during generation~\citep{informedcorrector,remedi,prism, GIDD}.
Other approaches model refinement at the sequence level, combine diffusion proposals with AR verification and rejection sampling~\citep{selfspec,tidar}, or employ self-correction~\citep{rediff,corrective-dlm,llada21}.
These methods show that refinement is a useful mechanism for improving draft quality and correcting local errors.
Nevertheless, most of them still operate within a predefined masking or corruption framework, so the refinement dynamics remain tied to hand-designed intermediate states.
By contrast, \method treats model-generated drafts themselves as the intermediate states and learns a forward-free refinement operator.
\vspace{-4pt}

\paragraph{Efficient decoding for language models}
Improving inference efficiency is a central challenge for language models, especially when generation requires many sequential forward passes.
For AR models, a large body of work accelerates decoding through multi-token prediction, speculative drafting, verification, and lightweight draft models~\citep{mtp-parallel,deepseek-v3,eagle1,eagle2,eagle3}.
Diffusion language models offer a different opportunity: they can update multiple tokens per forward pass, but this parallelism often introduces a quality--speed trade-off.
Prior work improves diffusion decoding by controlling confidence thresholds~\citep{diff-speed-conf, fast-dllm}, prediction entropy~\citep{eb-sampler}, or sampling policies~\citep{seed-diffusion, trace, credit-decoding}.
Other methods transfer diffusion-forcing ideas to discrete generation~\citep{diffusion-forcing,D2F}, approximate KV-cache reuse~\citep{fast-dllm,fast-dllm2, dKV_Cache}, or combine diffusion proposals with autoregressive likelihood correction~\citep{selfspec,tidar, dflash}.
These approaches primarily accelerate or stabilize inference under an existing diffusion or hybrid decoding process.
\method further improves the underlying refinement formulation itself, aligning training with the drafts used at inference time.

\section{Conclusion}
\label{sec:conclusion}

We introduced \method, a forward-free diffusion language model that learns generation through recursive refinement rather than a predefined corruption path.
By using model-generated drafts as intermediate states, \method aligns training with sampling dynamics while avoiding hand-designed discrete perturbations.
The framework is neighborhood-agnostic, sampler-adaptive, and compatible with flexible parameterizations including self-refinement and Best-of-$N$ refinement.
The repeated use of the shared refinement backbone also makes \method a draft-space Looped Transformer; because preceding refinement passes are detached, training avoids BPTT across refinement iterations while retaining ordinary backpropagation within the final pass.
Empirically, \method-4B outperforms larger 7--8B diffusion baselines while achieving a 1.5--1.8$\times$ average speedup.
Overall, forward-free recursive refinement provides a promising direction for efficient diffusion language modeling and a new way to realize looped computation over explicit text drafts.

\clearpage
\bibliography{ref,looped_refs}
\bibliographystyle{abbrv}

\clearpage
\appendix
\section{Theoretical Derivations}\label{app:proofs}

\subsection{Proof of Proposition~\ref{prop:joint_surrogate}}

\begin{proof}
The marginal refinement objective is
\begin{equation*}
    \KL\!\rbr{
        p_0\rbr{\dv}
        \,\|\,
         {  \int}
    K_{\theta,t}\rbr{\dv\mid\dv'}
    p_{\theta,t}\rbr{\dv'}
    \,d\dv'
    }
    =
    {  \int}
    p_0\rbr{\dv}
    \log
    \frac{
        p_0\rbr{\dv}
    }{
        {  \int}
        K_{\theta,t}\rbr{\dv\mid\dv'}
        p_{\theta,t}\rbr{\dv'}
        \,d\dv'
    }
    \,d\dv .
\end{equation*}
By Jensen's inequality and the concavity of $\log$, for every fixed $\dv$,
\begin{equation*}
    {  \int}
    q\rbr{\dv'|\dv}
    \log \frac{K_{\theta,t}\rbr{\dv\mid\dv'}
    p_{\theta,t}
    \rbr{\dv'}}{{q\rbr{\dv'|\dv}}}
    \,d\dv'
    \le
    \log
    {  \int}
    {q\rbr{\dv'|\dv}}\frac{K_{\theta,t}\rbr{\dv\mid\dv'}
    p_{\theta,t}
    \rbr{\dv'}}{{q\rbr{\dv'|\dv}}}
    \,d\dv' .
\end{equation*}
Multiplying by $p_0\rbr{\dv}$ and integrating over $\dv$ gives
\begin{equation*}
    \KL\!\rbr{
        p_0\rbr{\dv}
        \,\middle\|\,
        {  \int}
        K_{\theta,t}\rbr{\dv\mid\dv'}
        p_{\theta,t}\rbr{\dv'}
        \,d\dv'
    }
    \le
    \KL\!\rbr{
        p_0\rbr{\dv}q(\dv'|\dv)
        \,\middle\|\,
        K_{\theta,t}\rbr{\dv\mid\dv'}p_{\theta,t}\rbr{\dv'}
    }.
\end{equation*}
This proves the desired upper bound.
\end{proof}

\subsection{Derivation of Refinement Objective in Eq.~\eqref{eq:conditional_refinement_loss}}

We show that minimizing the joint surrogate reduces to conditional negative log-likelihood training. 
We have proven that the original matching objective in Eq.~\eqref{eq:ideal_refinement} is upper-bounded by the joint surrogate
\begin{equation*}
    \KL\!\rbr{
        p_0\rbr{\dv}p_{\theta,t}\rbr{\dv'}
        \,\middle\|\,
        K_{\theta,t}\rbr{\dv\mid\dv'}p_{\theta,t}\rbr{\dv'}
    }.
\end{equation*}
Expanding the KL divergence gives
\begin{align*}
    &
    \KL\!\rbr{
        p_0\rbr{\dv}p_{\theta,t}\rbr{\dv'|\dv}
        \,\middle\|\,
        K_{\theta,t}\rbr{\dv\mid\dv'}p_{\theta,t}\rbr{\dv'}
    }
    \nonumber\\
    &=
    {  \int}
    p_0\rbr{\dv}p_{\theta,t}\rbr{\dv'|\dv}
    \log
    \frac{
        p_0\rbr{\dv}p_{\theta,t}\rbr{\dv'|\dv}
    }{
        K_{\theta,t}\rbr{\dv\mid\dv'}p_{\theta,t}\rbr{\dv'}
    }
    \,d\dv\,d\dv'
    \nonumber\\
    &=
    {  \int}
    p_0\rbr{\dv}p_{\theta,t}\rbr{\dv'}
    \log \rbr{p_0\rbr{\dv}p_{\theta,t}\rbr{\dv'}}
    \,d\dv\,d\dv'
    -
    {  \int}
    p_0\rbr{\dv}p_{\theta,t}\rbr{\dv'|\dv}
    \log \rbr{K_{\theta,t}\rbr{\dv\mid\dv'}p_{\theta,t}\rbr{\dv'}}
    \,d\dv\,d\dv'
    \nonumber\\
    &
    \propto
    -
    \EE_{\dv\sim p_0,\,\dv'\sim p_{\theta,t}}
    \sbr{
        \log K_{\theta,t}\rbr{\dv\mid\dv'}
    }.
\end{align*}
Therefore, minimizing the joint surrogate over $K_{\theta,t}$ is equivalent to minimizing
\begin{equation*}
    \mathcal{L}_{t}\rbr{\theta}
    =
    -
    \EE_{\dv\sim p_0,\,\dv'\sim p_{\theta,t}}
    \sbr{
        \log K_{\theta,t}\rbr{\dv\mid\dv'}
    }.
\end{equation*}
This proves Eq.~\eqref{eq:conditional_refinement_loss}.

\subsection{Proof of Proposition~\ref{prop:bon_transition}}
\label{app:proof_bon_transition}

\begin{proof}
Given a draft $\dv'$, the self-refinement proposer draws $N$ candidates independently,
\begin{equation*}
    \dv_1,\ldots,\dv_N
    \sim
    K_{\theta}^{\mathrm{SR}}\rbr{\cdot\mid\dv'}.
\end{equation*}
Conditioned on the candidate set $\cbr{\dv_i}_{i=1}^{N}$, the Gumbel-Max selection rule selects candidate $j$ with probability
\begin{equation*}
    \PP\rbr{
        J=j
        \mid
        \cbr{\dv_i}_{i=1}^{N}
    }
    =
    \frac{
        \exp\rbr{\tau\scorer\rbr{\dv_j}}
    }{
        \sum_{\ell=1}^{N}
        \exp\rbr{\tau\scorer\rbr{\dv_\ell}}
    }.
    \label{eq:app_bon_selection_prob}
\end{equation*}
We first compute the contribution of the event that the first candidate equals a given sequence $\dv$ and is selected. 
Conditioning on $\dv_1=\dv$ and integrating over the remaining candidates gives
\begin{equation}
\begin{aligned}
    &
    \PP\rbr{
        \dv_1=\dv 
        \mid
        \dv'
    }
    \\
    &=
    { \int} \PP\rbr{\texttt{$\xb_1$ is the best among $\cbr{\dv_i}_{i=1}^N$}\mid \xb',\cbr{\dv_i}_{i=1}^N} K_{\theta,\mathrm{SR}}\rbr{\cbr{\dv_i}_{i=1}^N\mid\xb'} \,d\dv_{2:N}\\
    &=
    K_{\theta}^{\mathrm{SR}}\rbr{\dv\mid\dv'}
    {\int}
    \prod_{i=2}^{N}
    K_{\theta}^{\mathrm{SR}}\rbr{\dv_i\mid\dv'}
    \frac{
        \exp\rbr{\tau\scorer\rbr{\dv}}
    }{
        \exp\rbr{\tau\scorer\rbr{\dv}}
        +
        \sum_{j=2}^{N}
        \exp\rbr{\tau\scorer\rbr{\dv_j}}
    }
    \,d\dv_{2:N}
    \\
    &=
    K_{\theta}^{\mathrm{SR}}\rbr{\dv\mid\dv'}
    \EE_{\dv_{2:N}\sim K_{\theta}^{\mathrm{SR}}\rbr{\cdot\mid\dv'}}
    \sbr{
        \frac{
            \exp\rbr{\tau\scorer\rbr{\dv}}
        }{
            \exp\rbr{\tau\scorer\rbr{\dv}}
            +
            \sum_{j=2}^{N}
            \exp\rbr{\tau\scorer\rbr{\dv_j}}
        }
    }.
\end{aligned}
\label{eq:app_first_candidate_contribution}
\end{equation}
Therefore, Eq.~\eqref{eq:app_first_candidate_contribution} gives the induced best-of-$N$ refinement transition,
\begin{equation}
\begin{aligned}
    K_{\theta,\psi}^{\mathrm{BoN}}\rbr{\dv\mid\dv'}
    \propto K_{\theta}^{\mathrm{SR}}\rbr{\dv\mid\dv'}
    \EE_{\dv_{2:N}\sim K_{\theta}^{\mathrm{SR}}\rbr{\cdot\mid\dv'}}
    \sbr{
        \frac{
            \exp\rbr{\tau\scorer\rbr{\dv}}
        }{
            \exp\rbr{\tau\scorer\rbr{\dv}}
            +
            \sum_{j=2}^{N}
            \exp\rbr{\tau\scorer\rbr{\dv_j}}
        }
    }.
\end{aligned}
\end{equation}
This proves Proposition~\ref{prop:bon_transition}.
\end{proof}

\subsection{Derivation of Eq.~\eqref{eq:bon_loss}}
\label{app:derive_bon_loss}

We derive the tractable scorer objective by plugging the best-of-$N$ refinement transition into the conditional refinement loss. 
Let $\dv_1\sim p_0$ denote the positive target sequence, and let $\dv_2,\ldots,\dv_N$ be proposal candidates sampled from $K_{\theta}^{\mathrm{SR}}\rbr{\cdot\mid\dv'}$. 
From Eq.~\eqref{eq:app_first_candidate_contribution}, the first-candidate contribution to the best-of-$N$ transition is
\begin{equation*}
\begin{aligned}
    &
    K_{\theta}^{\mathrm{SR}}\rbr{\dv_1\mid\dv'}
    \EE_{\dv_{2:N}\sim K_{\theta}^{\mathrm{SR}}\rbr{\cdot\mid\dv'}}
    \sbr{
        \frac{
            \exp\rbr{\tau\scorer\rbr{\dv_1}}
        }{
            \exp\rbr{\tau\scorer\rbr{\dv_1}}
            +
            \sum_{j=2}^{N}
            \exp\rbr{\tau\scorer\rbr{\dv_j}}
        }
    } .
\end{aligned}
\label{eq:app_positive_transition_contribution}
\end{equation*}
Both $N$ and $K_{\theta}^{\mathrm{SR}}\rbr{\dv_1\mid\dv'}$ are independent of the scorer parameters during scorer training, since the proposer is treated with stop-gradient. 
Thus, the scorer-dependent part of the conditional refinement loss in Eq.~\eqref{eq:conditional_refinement_loss} is
\begin{equation}
\begin{aligned}
    -
    \EE_{\dv_1\sim p_0}\sbr{
    \log
    \EE_{\dv_{2:N}\sim K_{\theta}^{\mathrm{SR}}\rbr{\cdot\mid\dv'}}
    \sbr{
        \frac{
            \exp\rbr{\tau\scorer\rbr{\dv_1}}
        }{
            \exp\rbr{\tau\scorer\rbr{\dv_1}}
            +
            \sum_{j=2}^{N}
            \exp\rbr{\tau\scorer\rbr{\dv_j}}
        }
    }
    }.
\end{aligned}
\label{eq:app_bon_log_expectation}
\end{equation}
To obtain a tractable objective, we apply Jensen's inequality. 
Since $-\log(\cdot)$ is convex, for any positive random variable $R$,
\begin{equation*}
    -\log \EE\sbr{R}
    \le
    \EE\sbr{-\log R}.
    \label{eq:app_jensen_bon}
\end{equation*}
Using
\begin{equation*}
    R
    =
    \frac{
        \exp\rbr{\tau\scorer\rbr{\dv_1}}
    }{
        \exp\rbr{\tau\scorer\rbr{\dv_1}}
        +
        \sum_{j=2}^{N}
        \exp\rbr{\tau\scorer\rbr{\dv_j}}
    },
\end{equation*}
Eq.~\eqref{eq:app_bon_log_expectation} is upper-bounded by
\begin{equation*}
\begin{aligned}
    -
    \EE_{t,\,\dv_1\sim p_0,\,\cbr{\dv_j}_{j=2}^{N}\sim K_{\theta}^{\mathrm{SR}}\rbr{\cdot\mid\dv'}}\sbr{
    \log
    \frac{
        \exp\rbr{\tau\scorer\rbr{\dv_1}}
    }{
        \exp\rbr{\tau\scorer\rbr{\dv_1}}
        +
        \sum_{j=2}^{N}
        \exp\rbr{\tau\scorer\rbr{\dv_j}}
    }}.
\end{aligned}
\end{equation*}
Therefore, dropping constants independent of the scorer gives
\begin{equation*}
\begin{aligned}
  &\Lcal_{\mathrm{BoN}}\rbr{\psi}\\
  =&
  \EE_{t,\,\dv_1\sim p_0,\,\cbr{\dv_j}_{j=2}^{N}\sim K_{\theta}^{\mathrm{SR}}\rbr{\cdot\mid \dv'}}
  \sbr{
    -\tau\scorer\rbr{\dv_1}
    +
    \log
    \left(
        \exp\rbr{\tau\scorer\rbr{\dv_1}}
        +
        \sum_{j=2}^{N}
        \exp\rbr{\tau\scorer\rbr{\dv_j}}
    \right)
  }.
\end{aligned}
\end{equation*}
This is Eq.~\eqref{eq:bon_loss}.

\subsection{Properties and Justifications of Recursive Marginal Refinement}
\label{app:refinement_justification}

Although the learned sampler has the backward-chain form
\begin{equation*}
    p_\theta\rbr{\dv_{0:T}}
    =
    p_{\theta,T}\rbr{\dv_T}
    \prod_{t=1}^{T}
    K_{\theta,t}\rbr{\dv_{t-1}\mid\dv_t},
\end{equation*}
its learning procedure differs from conventional diffusion because the intermediate states are sampled from the model rather than generated by a predefined forward corruption path. 
We justify the recursive marginal refinement objective through two properties.

\begin{proposition}[Marginal consistency]
\label{prop:marginal_consistency_independent}
Let $\gamma_t\rbr{\dv,\dv'}=p_0\rbr{\dv}p_{\theta,t}\rbr{\dv'}$ be the product joint used in the refinement surrogate. 
If a refinement kernel $K$ satisfies
\begin{equation*}
    \KL\!\rbr{
        p_0\rbr{\dv}p_{\theta,t}\rbr{\dv'}
        \,\middle\|\,
        K\rbr{\dv\mid\dv'}p_{\theta,t}\rbr{\dv'}
    }
    =
    0,
\end{equation*}
then the induced refined marginal matches the target distribution:
\begin{equation*}
    {\int}
    K\rbr{\dv\mid\dv'}p_{\theta,t}\rbr{\dv'}\,d\dv'
    =
    p_0\rbr{\dv}.
\end{equation*}
More generally, if the joint surrogate is small, then the marginal refinement objective is also small:
\begin{equation*}
    \KL\!\rbr{
        p_0
        \,\middle\|\,
        Kp_{\theta,t}
    }
    \le
    \KL\!\rbr{
        p_0\rbr{\dv}p_{\theta,t}\rbr{\dv'}
        \,\middle\|\,
        K\rbr{\dv\mid\dv'}p_{\theta,t}\rbr{\dv'}
    }.
\end{equation*}
\end{proposition}

\begin{proof}
If the joint surrogate equals zero, then the two joint distributions are equal almost everywhere:
\begin{equation*}
    p_0\rbr{\dv}p_{\theta,t}\rbr{\dv'}
    =
    K\rbr{\dv\mid\dv'}p_{\theta,t}\rbr{\dv'}.
\end{equation*}
Integrating both sides over $\dv'$ gives
\begin{align*}
    {\int}
    p_0\rbr{\dv}p_{\theta,t}\rbr{\dv'}\,d\dv'
    &=
    {\int}
    K\rbr{\dv\mid\dv'}p_{\theta,t}\rbr{\dv'}\,d\dv' .
\end{align*}
Since $p_{\theta,t}$ is normalized, the left-hand side is $p_0\rbr{\dv}$. 
Therefore,
\begin{equation*}
    p_0\rbr{\dv}
    =
    {\int}
    K\rbr{\dv\mid\dv'}p_{\theta,t}\rbr{\dv'}\,d\dv',
\end{equation*}
which proves exact marginal consistency.
\end{proof}

\begin{proposition}[Ideal monotonic refinement]
\label{prop:ideal_monotonicity}
Let $\mathcal{K}$ be a refinement family that contains the identity kernel
\begin{equation*}
    I\rbr{\dv\mid\dv'}=\mathbf{1}\rbr{\dv=\dv'}.
\end{equation*}
Given the current draft marginal $p_{\theta,t}$, define the ideal refinement kernel as
\begin{equation*}
    K_t^\star
    =
    \arg\min_{K\in\mathcal{K}}
    \KL\!\rbr{
        p_0\rbr{\dv}
        \,\|\,
        {\int}
        K\rbr{\dv\mid\dv'}
        p_{\theta,t}\rbr{\dv'}
        \,d\dv'
    }.
\end{equation*}
Let the corresponding ideal refined marginal be
\begin{equation*}
    p_{\theta,t-1}^{\star}\rbr{\dv}
    =
    {\int}
    K_t^\star\rbr{\dv\mid\dv'}
    p_{\theta,t}\rbr{\dv'}
    \,d\dv'.
\end{equation*}
Then the ideal refinement step cannot increase the marginal KL to the target distribution:
\begin{equation*}
    \KL\!\rbr{
        p_0\rbr{\dv}
        \,\|\,
        p_{\theta,t-1}^{\star}\rbr{\dv}
    }
    \le
    \KL\!\rbr{
        p_0\rbr{\dv}
        \,\|\,
        p_{\theta,t}\rbr{\dv}
    }.
\end{equation*}
\end{proposition}

\begin{proof}
By definition of $K_t^\star$, we have
\begin{align*}
    \KL\!\rbr{
        p_0\rbr{\dv}
        \,\|\,
        p_{\theta,t-1}^{\star}\rbr{\dv}
    } \nonumber 
    =
    \min_{K\in\mathcal{K}}
    \KL\!\rbr{
        p_0\rbr{\dv}
        \,\|\,
        {\int}
        K\rbr{\dv\mid\dv'}
        p_{\theta,t}\rbr{\dv'}
        \,d\dv'
    }.
\end{align*}
Since $\mathcal{K}$ contains the identity kernel $I\rbr{\dv\mid\dv'}=\mathbf{1}\rbr{\dv=\dv'}$, the minimization over $\mathcal{K}$ includes the choice $K=I$. 
Therefore,
\begin{align*}
    \min_{K\in\mathcal{K}}
    \KL\!\rbr{
        p_0\rbr{\dv}
        \,\|\,
        {\int}
        K\rbr{\dv\mid\dv'}
        p_{\theta,t}\rbr{\dv'}
        \,d\dv'
    } \nonumber \\
   \le
    \KL\!\rbr{
        p_0\rbr{\dv}
        \,\|\,
        {\int}
        I\rbr{\dv\mid\dv'}
        p_{\theta,t}\rbr{\dv'}
        \,d\dv'
    }.
\end{align*}
The identity kernel preserves the current marginal:
\begin{equation*}
    {\int}
    I\rbr{\dv\mid\dv'}
    p_{\theta,t}\rbr{\dv'}
    \,d\dv'
    =
    p_{\theta,t}\rbr{\dv}.
\end{equation*}
Thus,
\begin{equation*}
    \KL\!\rbr{
        p_0\rbr{\dv}
        \,\|\,
        p_{\theta,t-1}^{\star}\rbr{\dv}
    }
    \le
    \KL\!\rbr{
        p_0\rbr{\dv}
        \,\|\,
        p_{\theta,t}\rbr{\dv}
    },
\end{equation*}
which proves the claim.
\end{proof}

\subsection{Interpretation of Empirical Best-of-$N$ Selection}
\label{app:log_ratio}

Let the proposal marginal induced by the current draft distribution and self-refinement proposer be
\begin{equation*}
    p_{\theta,t}^{\mathrm{SR}}\rbr{\dv}
    =
    {\int}
    K_{\theta}^{\mathrm{SR}}\rbr{\dv\mid\dv'}
    p_{\theta,t}\rbr{\dv'}
    \,d\dv'.
    \label{eq:app_sr_marginal}
\end{equation*}
The finite-candidate ranking objective contrasts positives $\dv_1\sim p_0$ with proposal samples $\dv_j\sim p_{\theta,t}^{\mathrm{SR}}$. 
At the population optimum, the score logit satisfies
\begin{equation*}
    \tau\scorer^\star\rbr{\dv}
    =
    \log p_0\rbr{\dv}
    -
    \log p_{\theta,t}^{\mathrm{SR}}\rbr{\dv}
    +
    C,
    \label{eq:app_marginal_log_ratio}
\end{equation*}
where $C$ is independent of $\dv$. 
Thus, the scorer estimates how much more likely a sequence is under the target distribution than under the proposal marginal.

This also motivates the empirical selection rule used in implementation. 
For candidates generated from a particular draft $\dv'$, ranking by
\begin{equation*}
    \log K_{\theta}^{\mathrm{SR}}\rbr{\dv\mid\dv'}
    +
    \tau\scorer\rbr{\dv}
\end{equation*}
combines the draft-conditioned proposal likelihood with a marginal correction toward the target distribution. 
This preserves the local preference of the proposer while using the scorer to favor globally more data-like candidates.
\section{Experiment Details}

\subsection{Baseline Models}

We compare \method with representative diffusion language models in the sub-8B regime as follows: 

\textbf{DiffuLLaMA}~\citep{DiffuLLaMA} is a 7B-parameter diffusion language model adapted from the autoregressive LLaMA2 backbone. 
It bridges autoregressive and diffusion-style training by modifying the attention mechanism and mitigating the causal-masking bias inherited from AR pretraining.

\textbf{LLaDA-MoE-7B-A1B-Base}~\citep{llada-moe} is a sparse mixture-of-experts diffusion language model with 7B total parameters and approximately 1B active parameters. 
It combines masked diffusion modeling with MoE routing to improve model capacity while keeping the inference-time active-parameter budget relatively small.

\textbf{LLaDA-8B-Base}~\citep{LLaDA} is a large masked diffusion language model that uses a forward masking process and a learned reverse generation process. 
It is parameterized by a Transformer denoiser that predicts masked tokens and is trained by optimizing a likelihood-based diffusion objective.

\textbf{Dream-7B-Base}~\citep{dream7b} is a 7B-parameter diffusion language model that refines sequences through parallel iterative denoising. 
It is initialized from an autoregressive language model and further trained with diffusion-style objectives and context-adaptive token-level noise scheduling.

\textbf{TiDAR-8B}~\citep{tidar} is an 8B-parameter hybrid model that combines diffusion-based draft generation with autoregressive-style verification or correction. 
It uses structured attention mechanisms to balance drafting and verification, and we report both its trusting-AR and trusting-diffusion variants when available. Since this model has not been open-sourced yet, all evaluation results are sourced from the original paper~\citep{tidar}.

\textbf{BlockDiff-4B}~\citep{block_diffusion} follows the block diffusion language modeling paradigm. 
It generates text autoregressively at the block level, while modeling the conditional distribution within each block using discrete diffusion. 
We train the BlockDiff-4B baseline using the same training setups as \method, following the blockwise masked diffusion formulation in~\citep{block_diffusion}.

\subsection{Evaluation Benchmarks}

We evaluate all models on six benchmarks covering general knowledge, mathematical reasoning, and code generation. 
Unless otherwise specified, we follow the shot settings described in the main text.

\subsubsection{General Knowledge}

\textbf{MMLU}~\citep{mmlu} is a multitask multiple-choice benchmark covering 57 subjects, including mathematics, history, computer science, law, medicine, and social sciences. 
In our evaluation, we use the standard test split with 14{,}042 questions.

\textbf{GPQA-Diamond}~\citep{gpqa} is the most challenging 198-question subset of GPQA, consisting of expert-written multiple-choice questions in biology, physics, and chemistry. 
Following our evaluation loader, we reserve the first 5 examples as in-context exemplars and evaluate on the remaining 193 questions.

\subsubsection{Mathematical Reasoning}

\textbf{GSM8K}~\citep{gsm8k} contains grade-school-level math word problems that require multi-step arithmetic reasoning. 
We evaluate on the full HuggingFace \texttt{openai/gsm8k} \texttt{main/test} split, which contains 1{,}319 questions.

\textbf{MATH-500}~\citep{math500} is a commonly used 500-problem subset of the MATH benchmark, covering competition-style mathematical reasoning across algebra, geometry, number theory, counting and probability, and related topics.

\subsubsection{Code Generation}

\textbf{HumanEval}~\citep{humaneval} evaluates functional code generation using 164 hand-written Python programming problems. 
Each problem includes a function signature, docstring, reference solution, and unit tests, and model outputs are evaluated by execution-based pass rates.

\textbf{HumanEval+}~\citep{humaneval_plus} extends HumanEval with substantially more comprehensive test cases and corrected evaluation cases. 
It provides a stricter execution-based evaluation protocol for assessing the functional correctness of generated code.
\section{Implementation Details of \method}
\label{app:implementation}

\subsection{Attention Mask Design}
\begin{figure}[!ht]
    \centering
    \includegraphics[width=\linewidth]{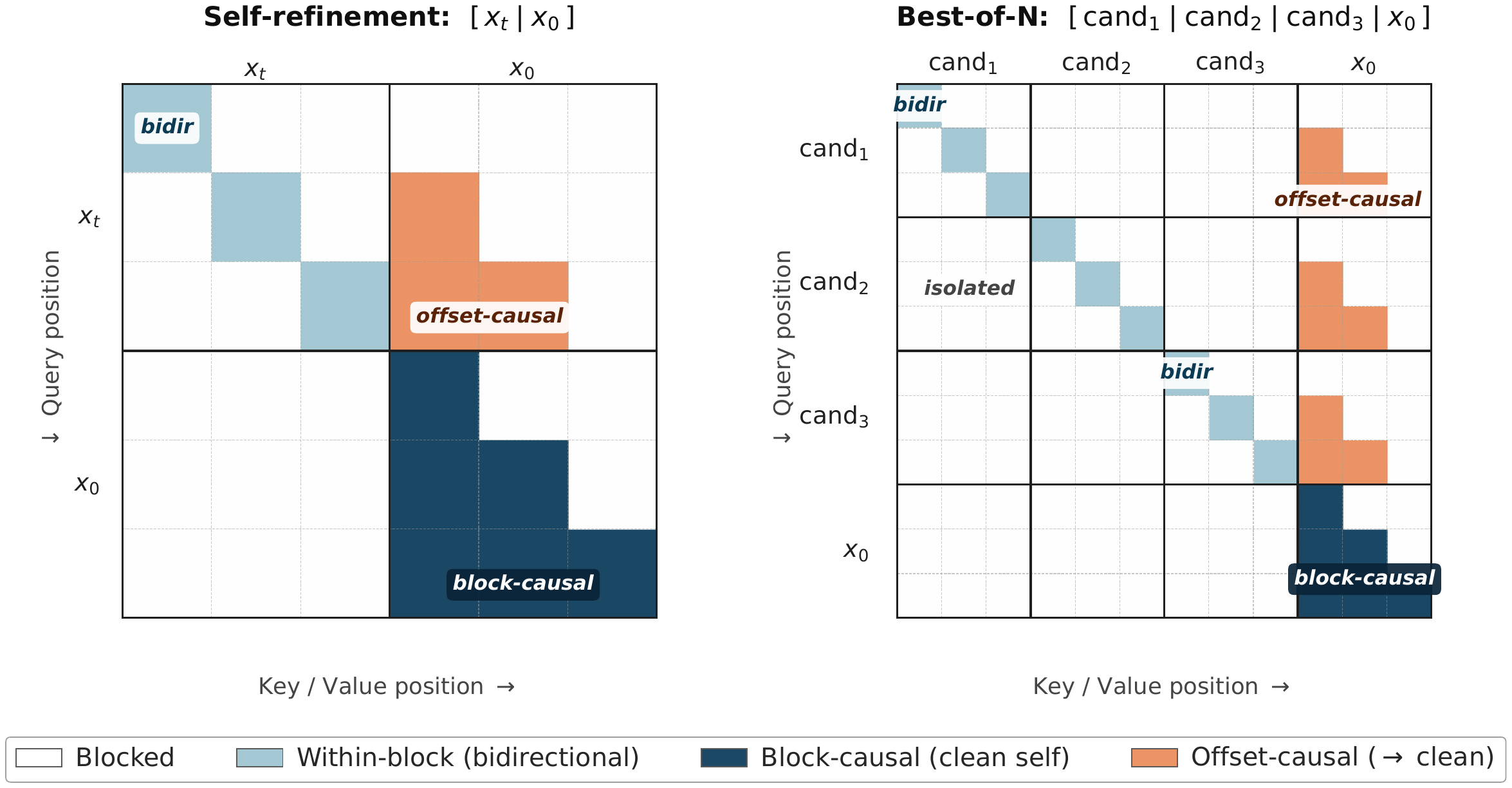}
    \caption{Attention mask design for \method's Self-refinement and Best-of-N forwards.}
    \label{fig:attn-mask-design}
\end{figure}

Figure~\ref{fig:attn-mask-design} illustrates the attention masks used for self-refinement and Best-of-$N$ refinement. We adopt a similar mask design for self-refinement following ~\citep{block_diffusion, sdar, llada20}.
In self-refinement, the input is organized as $[\dv_t \mid \dv_0]$. 
The refinement segment $\dv_t$ uses bidirectional attention within each block, allowing all positions in the current block to be jointly denoised, while the clean segment $\dv_0$ uses block-causal attention. 
In addition, refinement block $j$ can attend only to clean blocks $0,\ldots,j-1$, ensuring that the model conditions on committed prefix tokens without seeing the clean version of the block being refined. 
For Best-of-$N$ refinement, the input is organized as $[\mathrm{cand}_1\mid\cdots\mid\mathrm{cand}_N\mid\dv_0]$ and follows the same within-block bidirectional and offset-causal rules, but different candidate segments are mutually isolated. 
This isolation ensures that each candidate is scored only from its own tokens and the shared clean prefix, preventing the scorer from exploiting interactions among candidates. 
The bidirectional within-block attention enables global refinement and scoring within each block, while the block-causal structure preserves efficient left-to-right block generation.

\begin{wrapfigure}{r}{0.4\textwidth}
    \vspace{-4.0em}
    \centering
    \includegraphics[width=0.98\linewidth]{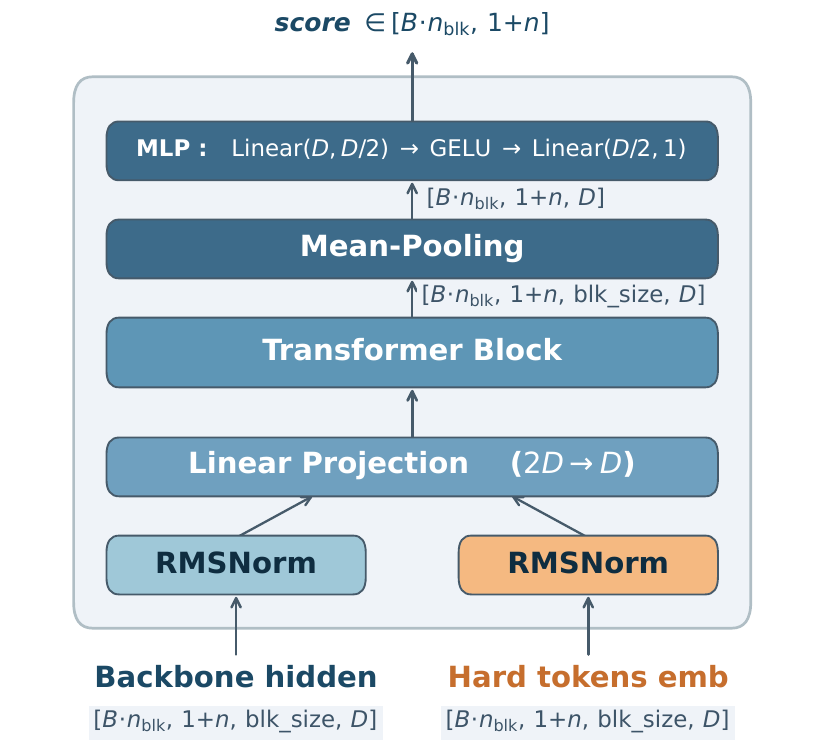}
    \caption{
    Best-of-$N$ scorer head with soft--hard fusion.
    }
    \label{fig:scorer-arch}
    \vspace{-1.0em}
\end{wrapfigure}

\subsection{Best-of-$N$ Scorer Architecture}
\label{app:scorer_arch}

The Best-of-$N$ variant of \method uses a lightweight scorer head to rank block-level candidate refinements. 
As shown in Fig.~\ref{fig:scorer-arch}, the scorer shares the self-refinement backbone, which is mostly frozen during scorer training, and predicts one scalar score for each candidate slot. 
Given $n$ candidate refinements and one clean positive slot, the scorer processes $(1+n)$ slots per block and is trained with the Best-of-$N$ objective, which treats the clean $x_0$ slot as the positive.

For each candidate, the scorer uses two complementary input channels. 
The soft channel is the frozen backbone hidden state $H$, which captures contextual compatibility under the blockwise attention pattern. 
The hard channel is the embedding of the candidate's discrete argmax tokens which preserves token-level identity information that may be weakened by the soft embedding. 
Inspired by the feature fusion adopted in multi-token prediction~(MTP) module in DeepSeek-V3~\citep{deepseek-v3}, the two channels are separately normalized and fused through a linear projection:
\[
    Z
    =
    W_{\mathrm{fuse}}
    \left[
        \mathrm{RMSNorm}(H)
        \,;\,
        \mathrm{RMSNorm}\rbr{E(\arg\max c)}
    \right].
\]
The fused representation is processed by a lightweight Qwen3-style Transformer block~\citep{qwen3}, masked mean pooling, and a two-layer MLP:
\begin{align*}
    s(c)
    &=
    \mathrm{MLP}
    \left(
        \mathrm{MeanPool}
        \rbr{
            \mathrm{TransformerBlock}(Z)
        }
    \right),
    \\
    \mathrm{MLP}
    &=
    \mathrm{Linear}(D,D/2)
    \rightarrow
    \mathrm{GELU}
    \rightarrow
    \mathrm{Linear}(D/2,1).
\end{align*}
The resulting $(1+n)$ scores are normalized with a softmax under the Best-of-$N$ objective. 
Because the scorer reuses the frozen proposer backbone and only adds a small head over pooled block representations, Best-of-$N$ selection remains efficient during both training and inference.

\subsection{Soft Prediction Embedding}
To represent sampler-induced drafts during recursive refinement, we adopt a soft prediction embedding from \citep{dmax}, rather than feeding back only hard argmax tokens. 
Let $h_i^{(t)}$ be the logits at position $i$ after refinement step $t$, and let
\begin{equation*}
    \hat{x}_i^{(t)}
    =
    \arg\max_{v\in\Vcal} h_{i,v}^{(t)}
\end{equation*}
denote the selected token. 
We construct the next-step input embedding as
\begin{equation}
    e_i^{(t+1)}
    =
    \alpha_i^{(t)}
    E\rbr{\hat{x}_i^{(t)}}
    +
    \rbr{1-\alpha_i^{(t)}}
    E\rbr{\blanktoken},
    \qquad
    \alpha_i^{(t)}
    =
    \mathrm{softmax}\rbr{h_i^{(t)}}_{\hat{x}_i^{(t)}},
    \label{eq:soft_prediction_embedding}
\end{equation}
where $E$ is the token embedding table and $\alpha_i^{(t)}$ is the model probability of the selected token. 
Thus, high-confidence selected tokens are represented close to their hard token embeddings, while uncertain predictions remain close to the mask embedding and can be further refined. 
This confidence-weighted interpolation is related to recent smoothing and trace-based decoding techniques for diffusion language models~\citep{smoothie,credit-decoding}. 
Unlike methods that apply such smoothing only during inference, \method uses the same soft prediction embedding during both training and inference, reducing the train--inference mismatch of intermediate drafts and improving refinement stability.

\begin{algorithm}[!htbp]
    \SetAlgoLined
    \DontPrintSemicolon
    \SetNoFillComment
    \SetNlSkip{0.4em}
    \LinesNotNumbered
    \linespread{.86}\selectfont
    \caption{\method Training}
    \label{alg:training}
    \SetKwInput{KwIn}{Input}
    \SetKwInput{KwOut}{Output}

    \KwIn{
    Training corpus $\Dcal$; refinement backbone $K_{\theta}$; optional scorer $\scorer$; \\
    maximum refinements $K_{\max}$; Best-of-$N$ width $n$; blockwise attention mask $\Mcal_{\mathrm{block}}$ from Fig.~\ref{fig:attn-mask-design}.
    }
    \KwOut{Trained self-refinement model $K_{\theta}$ and optional Best-of-$N$ scorer $\scorer$.}

    \tcc{Stage I: train sampler-induced self-refinement}
    Initialize $K_{\theta}$ from the base language model\;
    \For{training step}{
        Sample clean sequence $\dv_0\sim\Dcal$ and sample $K\le K_{\max}$\;
        Initialize draft $\dv_K\leftarrow[\blanktoken]$ for the current block\;
        \For{$r=K,\ldots,1$}{
            Form training input $\ub_r\leftarrow[\,\dv_r;\dv_0\,]$\;
            Run $K_{\theta}$ on $\ub_r$ with blockwise attention mask $\Mcal_{\mathrm{block}}$\;
            \eIf{$r>1$}{
                Update draft $\dv_{r-1}$ from the prediction with no gradient\;
                $\dv_{r-1}\leftarrow\mathrm{stopgrad}\rbr{\dv_{r-1}}$\;
            }{
                Update $\theta$ by minimizing the self-refinement loss in Eq.~\eqref{eq:self_refine_loss}\;
            }
        }
    }

\tcc{Stage II: train Best-of-$N$ scorer, if used}
\If{Best-of-$N$ refinement is enabled}{
    Freeze the refinement backbone $K_{\theta}$\;
    \For{training step}{
        Sample clean positive $\dv^{+}\sim\Dcal$ and sample $K\le K_{\max}$\;
        Generate $n$ candidate refinements $\{\dv_i^{-}\}_{i=1}^{n}$ by running the self-refinement sampler with frozen $K_{\theta}$\;
        Concatenate positive and candidates into one BoN input 
        $\ub_{\mathrm{bon}}\leftarrow[\,\dv^{+};\dv_1^{-};\cdots;\dv_n^{-}\,]$\;
        Run frozen $K_{\theta}$ on $\ub_{\mathrm{bon}}$ with BoN attention mask $\Mcal_{\mathrm{BoN}}$, which isolates candidate slots\;
        Score the $(1+n)$ slots with $\scorer$\;
        Update $\phi$ with the Best-of-$N$ objective in Eq.~\eqref{eq:bon_loss}\;
    }
}
\end{algorithm}

\subsection{Pseudocode for Training and Inference}
\label{app:pseudocode}

We provide pseudocode for the main training and decoding procedures of \method. 
Algorithm~\ref{alg:training} summarizes the two-stage training pipeline, including sampler-induced self-refinement training and optional Best-of-$N$ scorer training. 
Algorithm~\ref{alg:self-refine-infer} describes self-refinement decoding for a single block, where the model iteratively updates a draft and commits high-confidence tokens. 
Algorithm~\ref{alg:bon-one-block} extends this procedure with Best-of-$N$ candidate selection, and Algorithm~\ref{alg:beam-sampling} details the marginal-cost beam sampler used to construct diverse candidate blocks efficiently.

\begin{algorithm}[!htbp]
    \SetAlgoLined
    \DontPrintSemicolon
    \SetNoFillComment
    \LinesNotNumbered
    \linespread{.86}\selectfont
    \caption{Self-Refinement Decoding for One Block}
    \label{alg:self-refine-infer}
    \SetKwInput{KwIn}{Input}
    \SetKwInput{KwOut}{Output}

    \KwIn{Refinement backbone $K_{\theta}$; committed prefix $x_{\mathrm{pre}}$; block size $B$; refinement budget $K$; confidence threshold $\eta$.}
    \KwOut{Decoded block $\hat c\in\Vcal^B$.}

    Initialize $\hat c \leftarrow [\blanktoken]^B$ and commitment mask $m\leftarrow[0]^B$\;
    Initialize $p \leftarrow \mathrm{Uniform}(\Vcal)^B$\;

    \For{$k=1,\ldots,K$}{
        \eIf{$k=1$}{
            Form input $\ub\leftarrow[\,\hat c;x_{\mathrm{pre}}\,]$ using hard blank tokens\;
        }{
            Build soft input embeddings $e\leftarrow\textsc{SoftEmbed}(\hat c,p,m)$ using Eq.~\eqref{eq:soft_prediction_embedding}\;
            Form input $\ub\leftarrow[\,e;x_{\mathrm{pre}}\,]$\;
        }

        Compute $p \leftarrow \mathrm{softmax}\rbr{K_{\theta}(\ub;\Mcal_{\mathrm{block}})}_{1:B}$\;
        Set $\tilde c_i \leftarrow \arg\max_{v\in\Vcal} p_{i,v}$ and $\alpha_i\leftarrow p_{i,\tilde c_i}$ for all $i$\;

        Reset $m\leftarrow[0]^B$\;
        \tcc{Early commitment.}
        \For{$i=1,\ldots,B$}{
            \eIf{$\alpha_i\ge\eta$}{
                $\hat c_i\leftarrow\tilde c_i$ and $m_i\leftarrow1$\;
            }{
                \textbf{break}\;
            }
        }

        \tcc{Early stop if all positions in the block are committed.}
        \If{$\forall i,\ m_i=1$}{
            \textbf{break}\;
        }
    }

    \tcc{If the iteration budget is exhausted, force-commit remaining positions.}
    \For{each $i$ with $m_i=0$}{
        $\hat c_i\leftarrow \arg\max_{v\in\Vcal}p_{i,v}$\;
    }

    \Return{$\hat c$}\;
\end{algorithm}

\begin{algorithm}[!htpb]
    \SetAlgoLined
    \DontPrintSemicolon
    \SetNoFillComment
    \LinesNotNumbered
    \linespread{.86}\selectfont
    \caption{Best-of-$N$ Refinement Decoding for One Block}
    \label{alg:bon-one-block}
    \SetKwInput{KwIn}{Input}
    \SetKwInput{KwOut}{Output}

    \KwIn{
    Refinement backbone $K_{\theta}$; scorer $\scorer$; committed prefix $x_{\mathrm{pre}}$; \\
    block size $B$; refinement budget $K$; width $n$; commitment threshold $\eta$; scorer weight $\tau$.
    }
    \KwOut{Decoded block $\hat c\in\Vcal^B$.}

    Initialize draft $\hat c\leftarrow[\blanktoken]^B$, commitment mask $m\leftarrow[0]^B$, and $p\leftarrow\mathrm{Uniform}(\Vcal)^B$\;

    \For{$k=1,\ldots,K$}{
        \tcc{Generate $n$ candidate drafts from the current state.}
        Generate candidate drafts $\{\hat c_i\}_{i=1}^{n}$ from $\hat c$ using marginal-cost beam sampling or stochastic candidate sampling\;
        Compute proposal scores $\ell_i\leftarrow\log q_{\theta}\rbr{\hat c_i\mid \hat c,x_{\mathrm{pre}}}$ for $i=1,\ldots,n$\;

        \tcc{One packed BoN forward gives both refinement predictions and scorer values.}
        Form BoN input 
        $\ub_{\mathrm{bon}}\leftarrow[\,\hat c_1;\hat c_2;\cdots;\hat c_n;x_{\mathrm{pre}}\,]$\;
        Run frozen $K_{\theta}$ on $\ub_{\mathrm{bon}}$ with BoN attention mask $\Mcal_{\mathrm{BoN}}$\;
        \tcp{$\Mcal_{\mathrm{BoN}}$ isolates candidate slots while allowing each candidate to attend to the committed prefix.}
        Obtain backbone token probabilities $p_i$ and scorer values $s_i\leftarrow\scorer(\hat c_i,x_{\mathrm{pre}})$ for each candidate slot\;

        \tcc{Select the best candidate and continue refinement from it.}
        $i^\star\leftarrow
        \arg\max_{i\in\{1,\ldots,n\}}
        \rbr{\ell_i+\tau s_i}$\;
        $\hat c\leftarrow \hat c_{i^\star}$, \quad $p\leftarrow p_{i^\star}$\;

        \tcc{Early commitment under the selected candidate.}
        Set $\tilde c_j\leftarrow\arg\max_{v\in\Vcal}p_{j,v}$ and $\kappa_j\leftarrow p_{j,\tilde c_j}$ for all $j$\;
        Reset $m\leftarrow[0]^B$\;
        \For{$j=1,\ldots,B$}{
            \eIf{$\kappa_j\ge\eta$}{
                $\hat c_j\leftarrow\tilde c_j$ and $m_j\leftarrow1$\;
            }{
                \textbf{break}\;
            }
        }

        \tcc{Early stop if all positions are committed.}
        \If{$\forall j,\ m_j=1$}{
            \textbf{break}\;
        }
    }

    \tcc{If the iteration budget is exhausted, force-commit remaining positions.}
    \For{each $j$ with $m_j=0$}{
        $\hat c_j\leftarrow \arg\max_{v\in\Vcal}p_{j,v}$\;
    }

    \Return{$\hat c$}\;
\end{algorithm}

\begin{algorithm}[!htbp]
    \SetAlgoLined
    \DontPrintSemicolon
    \SetNoFillComment
    \LinesNotNumbered
    \linespread{.86}\selectfont
    \caption{Marginal-Cost Beam Candidate Sampling}
    \label{alg:beam-sampling}
    \SetKwInput{KwIn}{Input}
    \SetKwInput{KwOut}{Output}

    \KwIn{Per-position probabilities $p\in[0,1]^{B\times|\Vcal|}$; block size $B$; number of candidates $n$; maximum swap size $R$.}
    \KwOut{Ordered candidate blocks $(\hat c_1,\ldots,\hat c_n)$ and proposal scores $(\ell_1,\ldots,\ell_n)$.}

    \For{$j=1,\ldots,B$}{
        $v_1[j]\leftarrow\arg\max_{v\in\Vcal}p_{j,v}$\;
        $v_2[j]\leftarrow\arg\max_{v\in\Vcal\setminus\{v_1[j]\}}p_{j,v}$\;
        $\Delta[j]\leftarrow\log p_{j,v_1[j]}-\log p_{j,v_2[j]}$\;
    }

    $\Bcal\leftarrow\{(\emptyset,0)\}$\;

    \For{$r=1,\ldots,R$}{
        \For{each subset $S\subseteq\{1,\ldots,B\}$ with $|S|=r$}{
            $\Bcal\leftarrow\Bcal\cup\{(S,\sum_{j\in S}\Delta[j])\}$\;
        }
    }

    Sort $\Bcal$ by marginal cost and keep the $n$ cheapest subsets\;

    \For{$i=1,\ldots,n$}{
        $(S_i,\cdot)\leftarrow\Bcal[i]$\;
        $\hat c_i\leftarrow v_1$\;
        \For{each $j\in S_i$}{
            $\hat c_i[j]\leftarrow v_2[j]$\;
        }
        $\ell_i\leftarrow\sum_{j=1}^{B}\log p_{j,\hat c_i[j]}$\;
    }

    \Return{$(\hat c_1,\ldots,\hat c_n),(\ell_1,\ldots,\ell_n)$}\;
\end{algorithm}

\section{Training Details of \method}
\label{app:training_details}

\begin{table}[!htbp]
\centering
\caption{
Training and optimization details for \method. 
}
\label{tab:training_details}
\setlength{\tabcolsep}{5pt}
\renewcommand{\arraystretch}{1.12}
\begin{tabular}{ll}
\toprule
\textbf{Configuration} & \textbf{Value} \\
\midrule
Base model & Qwen3-4B-Base \\
Block size $B$ & 4 \\
Sequence length $S$ & 4096 \\
Mixed precision & bf16 \\
Gradient checkpointing & enabled \\
\midrule
Optimizer & AdamW \\
Learning rate & $5\times10^{-5}$, constant after warm-up \\
Warm-up & 250 steps, linear \\
Weight decay & 0.01 \\
Adam betas & $(0.9, 0.999)$ \\
Max gradient norm & 1.0 \\
EMA & disabled \\
Backbone LR scale & 1.0 \\
\texttt{set\_to\_none} & true \\
Random seed & 42 \\
\bottomrule
\end{tabular}
\end{table}

\subsection{Dataset Preparation}
We construct a curated 10B-token continued-pretraining corpus from four public data streams. 
The largest component is \textbf{nemotron\_math}, drawn from \texttt{nvidia/Nemotron-CC-Math-v1} with the \texttt{4plus\_MIND} configuration, using the \texttt{train} split and \texttt{text} field. 
This stream corresponds to the highest-quality MIND-classified bucket in NVIDIA's Nemotron-CC-Math corpus and serves as the primary source for mathematics and reasoning~\citep {nemotron-cc-math}. 
The \textbf{nemotron\_code} stream is drawn from \texttt{nvidia/Nemotron-Pretraining-Specialized-v1.1} with the \texttt{Nemotron-Pretraining-Code-Concepts} configuration, using the \texttt{train} split and \texttt{text} field. 
This stream provides algorithmic reasoning and code concept documents~\citep{nemotron-code}. 
The \textbf{nemotron\_logic} stream is also drawn from \texttt{nvidia/Nemotron-Pretraining-Specialized-v1.1}, using the \texttt{Nemotron-Pretraining-Formal-Logic} configuration, and provides formal-reasoning and logic data. 
Finally, the \textbf{fineweb\_edu} stream is drawn from \texttt{HuggingFaceFW/fineweb-edu} with the \texttt{sample-10BT} configuration, using the \texttt{train} split and \texttt{text} field. 
FineWeb-Edu is an educational subset of FineWeb filtered from large-scale web data and serves as a general educational stabilizer~\citep{fineweb-edu}.

The final mixture contains NVIDIA-Nemotron math~\citep{nemotron-cc-math}~(40\%, approximately 4.0B tokens), NVIDIA-Nemotron code-concepts~\citep{nemotron-code}~(36\%, approximately 3.6B tokens), FineWeb-Edu~\citep{fineweb-edu}~(22\%, approximately 2.2B tokens), and NVIDIA-Nemotron formal-logic~\citep{nemotron-code}~(2\%, approximately 0.2B tokens). 
The resulting corpus is intentionally STEM- and reasoning-oriented, with 78\% of the data coming from mathematics, code, and formal logic sources, aligning with our downstream focus on mathematical reasoning and code generation. 
We additionally decontaminate the training corpus against all evaluation benchmarks used in this work to reduce the risk of test-set leakage. 
All documents are tokenized with the Qwen3-4B-Base tokenizer, which has a vocabulary size of 151{,}936. 

We pack tokenized examples into fixed-length sequences of 4096 using neat-packing, where multiple shorter documents may be concatenated into a single training sequence. 
To prevent information leakage between unrelated documents during training, we maintain a per-token document identifier and apply document-segmented attention masks, ensuring that tokens cannot attend across document boundaries. 
The prepared corpus is stored as int32 token IDs and used for continued pretraining of the \method backbone.

\subsection{Training Schedule and Optimizer Setup}

We train \method in two stages. 
First, we continue pretraining the Qwen3-4B-Base backbone with the forward-free denoising cross-entropy objective for 8B tokens. 
To stabilize refinement learning and encourage the model to preserve already correct tokens, we randomly replace a small fraction of initial mask positions with their ground-truth tokens and also include these positions in the cross-entropy loss. 
The ground-truth replacement ratio is linearly annealed from 0.3 to 0.1 over the first 4B training tokens and is kept at 0.1 for the remainder of training. 
After the 8B-token warm-up, we branch into two variants for the remaining 2B tokens. 
The self-refinement variant continues optimizing the same denoising cross-entropy objective, while the best-of-N variant is trained with the NCE ranking objective. 
For NCE training, we freeze either the full backbone or all but the last Transformer layer and train the scorer module for stability; we select the final checkpoint from these settings based on validation performance.

\end{document}